\documentclass{article}

 \usepackage[preprint]{neurips_2026}

\usepackage[utf8]{inputenc} 
\usepackage[T1]{fontenc}    
\usepackage[colorlinks=true,linkcolor=purple,citecolor=blue,urlcolor=blue]{hyperref} 
\usepackage{url}            
\usepackage{booktabs}       
\usepackage{threeparttable}
\usepackage{amsfonts}       
\usepackage{nicefrac}       
\usepackage{microtype}      
\usepackage{amsmath} 
\usepackage{cleveref}       
\usepackage{xcolor}         
\usepackage[table]{xcolor}
\usepackage[dvipsnames]{xcolor}
\usepackage{multirow}

\usepackage{graphicx}
\usepackage{amsmath,amssymb} 
\usepackage{amsthm}
\usepackage{tcolorbox}
\usepackage{enumitem}
\usepackage{caption}
\usepackage{subfigure}
\usepackage{array}
\usepackage{balance}
\usepackage{colortbl}
\usepackage{stfloats}
\usepackage{tabularx}
\usepackage{textcomp}
\usepackage{verbatim}
\usepackage{bm}
\usepackage{colortbl}
\usepackage{wrapfig}
\usepackage{pifont} 
\usepackage{fontawesome5} 
\newcommand{\cmark}{\ding{51}}%
\newcommand{\xmark}{\ding{55}}%

\usepackage{algorithm}
\usepackage{algorithmic}

\definecolor{b_proj}{HTML}{D71D3D}
\definecolor{b_mse}{HTML}{FFD966}
\definecolor{b_act}{HTML}{999999}

\renewcommand{\paragraph}[1]{\noindent\textbf{#1.}}
\renewcommand{\subparagraph}[1]{\noindent\textbf{\underline{#1.}}}

\newtheorem{proposition}{Proposition}

\renewenvironment{proof}{\noindent{\bfseries Proof:}}{\qed \smallskip} 

\title{Continuous-Time Distribution Matching for Few-Step Diffusion Distillation}
\author{
  \textbf{Tao Liu$^{1}$, Hao Yan$^{2}$, Mengting Chen$^{2,}$\thanks{Project leader.}{}, Taihang Hu$^{2}$, Zhengrong Yue$^{2}$, Zihao Pan$^{2}$} \\
  \textbf{Jinsong Lan$^{2}$, Xiaoyong Zhu$^{2}$, Ming-Ming Cheng$^{1}$, Bo Zheng$^{2,}$\thanks{Co-corresponding authors.}{}, Yaxing Wang$^{3,\dagger}$} \\
  {\normalfont $^{1}$VCIP, College of Computer Science, Nankai University \quad\quad $^{2}$Alibaba Group} \\
  {\normalfont $^{3}$College of Artificial Intelligence, Jilin University} \\
    \\[0.2em]
  {\normalfont 
  \begin{tabular}{@{}l@{}}
    \faGlobe\ \ \href{https://byliutao.github.io/cdm_page/}{https://byliutao.github.io/cdm\_page/} \\
    \faGithub\ \ \href{https://github.com/byliutao/CDM}{https://github.com/byliutao/cdm}
  \end{tabular}
  }
}

\begin{document}
\maketitle

\begin{center}
    \captionsetup{hypcap=false} 
    \includegraphics[width=\textwidth]{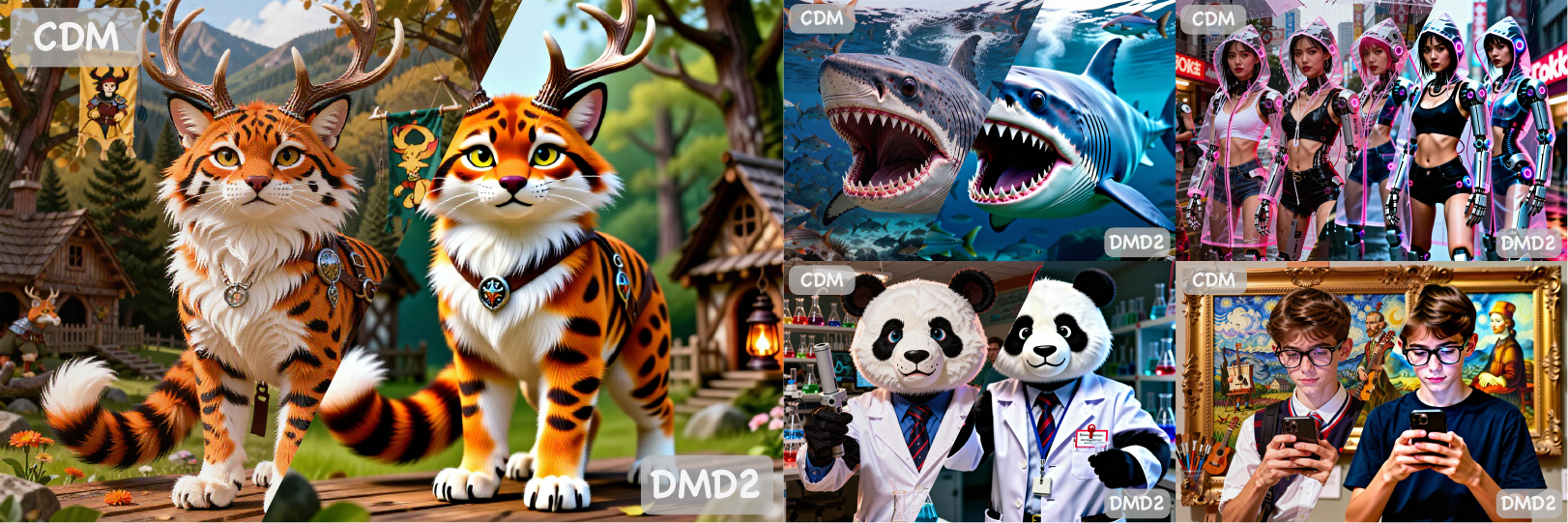} 
    \captionof{figure}{
        \textbf{CDM enables high-fidelity few-step text-to-image generation.}
        We compare our \emph{Continuous-Time Distribution Matching} (CDM) against DMD2, both distilled from Longcat-Image (1024$\times$1024) and evaluated at 4 NFE with identical prompts and seeds.
        Without relying on any GAN or reward-model auxiliary objectives, CDM produces sharper textures, richer fine-grained details, and overall higher visual fidelity, while DMD2 suffers from noticeable over-smoothing and detail loss.
        \emph{(Best viewed zoomed-in.)}
    }
    \label{fig:teaser}
\end{center}

\begin{abstract}
Step distillation has become a leading technique for accelerating diffusion models, among which Distribution Matching Distillation (DMD) and Consistency Distillation are two representative paradigms.
While consistency methods enforce self-consistency along the full PF-ODE trajectory to steer it toward the clean data manifold, vanilla DMD relies on sparse supervision at a few predefined discrete timesteps.
This restricted discrete-time formulation and mode-seeking nature of the reverse KL divergence tends to exhibit visual artifacts and over-smoothed outputs, often necessitating complex auxiliary modules---such as GANs or reward models---to restore visual fidelity.
In this work, \emph{we introduce \textbf{C}ontinuous-Time \textbf{D}istribution \textbf{M}atching (\textbf{CDM}), migrating the DMD framework from discrete anchoring to continuous optimization for the first time.}
CDM achieves this through two continuous-time designs.
First, we replace the fixed discrete schedule with a dynamic continuous schedule of random length, so that distribution matching is enforced at arbitrary points along sampling trajectories rather than only at a few fixed anchors.
Second, we propose a continuous-time alignment objective that performs active off-trajectory matching on latents extrapolated via the student's velocity field, improving generalization and preserving fine visual details.
Extensive experiments on different architectures, including SD3-Medium and Longcat-Image, demonstrate that CDM provides highly competitive visual fidelity for few-step image generation without relying on complex auxiliary objectives.
\end{abstract}
\section{Introduction}
\label{sec:intro}

The remarkable capabilities of diffusion and flow-matching models~\cite{esser2024scaling, ho2020denoising, lipmanflow, liuflow, rombach2022high, songdenoising} have revolutionized text-to-image generation in recent years, setting new benchmarks for high-fidelity visual synthesis.
Despite their exceptional generation quality, these models fundamentally rely on an iterative sampling process.
This sequential procedure, typically demanding tens to hundreds of network evaluations, imposes a severe computational bottleneck that ultimately limits their real-world deployment.
Accelerating this generation process without sacrificing sample quality has therefore become a central research challenge.

To bridge this gap, a variety of diffusion distillation paradigms have emerged~\cite{liu2023instaflow, luo2023latent, meng2023distillation, salimans2022progressive, sauer2024adversarial, song2023consistency}.
While early efforts reduced sampling to a few steps, the resulting models often struggle to balance inference speed with faithful text-image alignment.
Among the diverse technical routes aimed at few-step synthesis, score-based distribution matching---prominently represented by Diff-Instruct~\cite{luo2023diff} and Distribution Matching Distillation~\cite{yin2024one}---has emerged as a leading framework.
By mathematically matching the student's output distribution with the pre-trained teacher's target distribution, these methods have demonstrated state-of-the-art performance in accelerating generative models.

Despite its success, existing DMD methods~\cite{chadebec2025flash, liu2025decoupled, yin2024improved} inherit a structural limitation from their backward simulation strategy.
To keep the simulated training trajectory consistent with the few-step inference procedure, they restrict the simulated timesteps to a fixed set of discrete anchors that matches the inference schedule.
Unlike Consistency Distillation~\cite{lusimplifying,luo2023latent,song2023consistency}, which naturally optimizes trajectories within a continuous space, this strict confinement to sparse discrete schedules severely limits DMD.
The lack of intermediate, dense supervision forces the student to learn an unsmooth velocity field.
Furthermore, the underlying reverse KL objective is inherently mode-seeking~\cite{lu2025adversarial, xie2024distillation}, biasing the student toward a few dominant modes of the teacher's distribution.
Consequently, the generated images often suffer from oversmoothing and visual artifacts, typically necessitating complex auxiliary modules (such as GANs or reward models) to restore visual fidelity~\cite{chadebec2025flash, yin2024one}.

However, our preliminary empirical analysis challenges this strict training-inference alignment requirement~\cite{karras2022elucidating} (\Cref{fig:schedule}).
We investigate an alternative formulation where the model is optimized via backward simulation using uniformly sampled continuous timesteps $t \in (0, 1]$ with random length at each training iteration, decoupling it from the fixed inference schedule.
By simply randomizing the training timestep at each iteration, the student is trained over the full continuous time space rather than a few fixed points, and receives teacher gradients from a much wider range of trajectories.
Empirically, this simple change not only preserves distillation performance, but yields consistent improvements: the dynamically scheduled model attains higher HPSv3 scores with finer details and fewer artifacts than its strictly aligned counterpart.
This suggests that distribution matching is \textit{schedule-independent}---rather than serving as a necessary anchor, the discrete schedule acts as an overly restrictive constraint on the student's achievable quality.

\begin{figure}[htp]
    \centering
    \includegraphics[width=\textwidth]{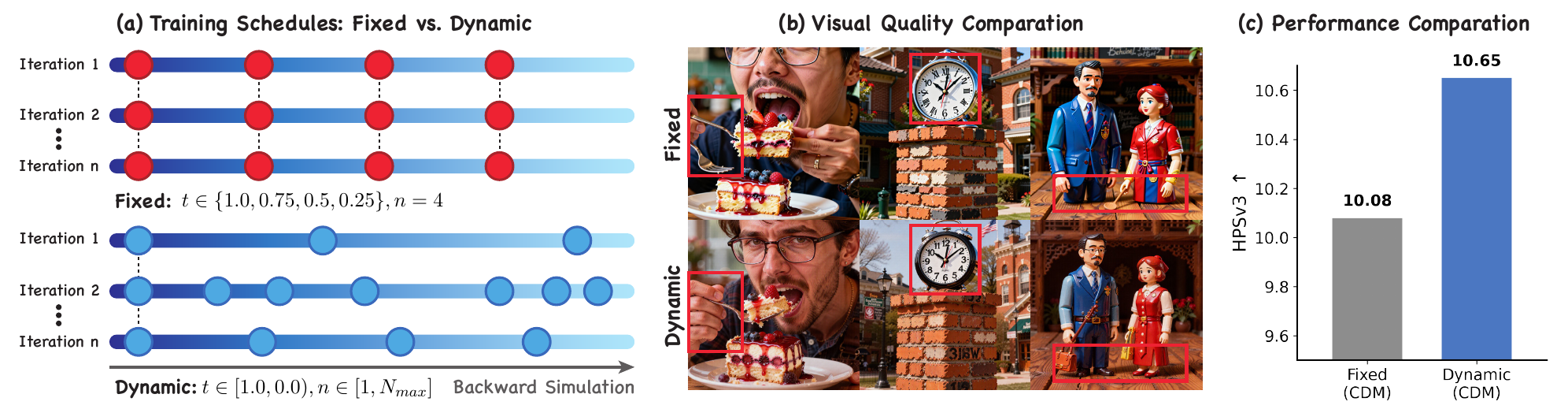}
    \vspace{-0.2cm}
    \caption{
        \textbf{Empirical evidence of schedule decoupling.} 
        \textbf{(a)} Conventional distillation strictly anchors backward simulation to predefined discrete inference timesteps. In contrast, our dynamic scheduling optimizes over uniformly sampled continuous timesteps $t \in (0, 1]$ at each iteration. 
        \textbf{(b)} Visually, the dynamically scheduled model produces finer details and fewer artifacts than the strictly aligned baseline. 
        \textbf{(c)} Quantitatively, it also attains a higher HPSv3 score, indicating that exact discrete alignment is not only unnecessary but in fact restrictive---motivating our continuous-time formulation.
    }
    \label{fig:schedule}
\end{figure}

Given that distribution matching benefits from unrestricted continuous timesteps, it is crucial to understand what exactly the model learns from these matching signals.
Recent studies~\cite{liu2025decoupled,yu2023text} decouple DMD training into a CFG Augmentation (CA) loss and a Distribution Matching (DM) loss, treating the latter simply as a "regularizer" for training stability and mitigating artifacts.
However, visual evidence in \Cref{fig:cfg_vs_nocfg} (further supported by the quantitative validation in Appendix \Cref{tab:dm_loss_analysis}) reveals a fundamentally different paradigm.
When student models are distilled solely with the DM loss, their generated images closely match the samples produced by the teacher \textit{without classifier-free guidance (CFG)}---which we refer to as the teacher's \textbf{CFG-free distribution}.
This tight correlation indicates that the achievable performance of the DM loss is closely aligned with the teacher's CFG-free distribution.
Rather than acting as a passive regularizer, the DM loss plays a substantive role in faithfully capturing this CFG-free distribution throughout the distillation process.

\begin{figure}[t]
    \centering
    \includegraphics[width=\textwidth]{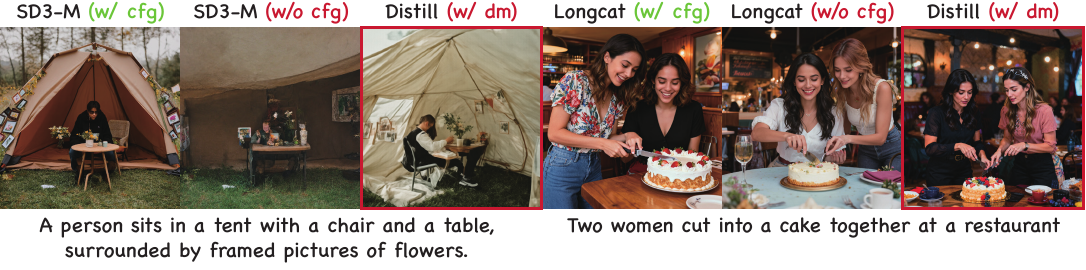}
    \caption{\textbf{Visual evidence on the role of the DM loss.} Samples from teacher models (SD3-Medium and Longcat-Image) with and without CFG, compared against student models distilled with the DM loss alone. Students distilled with the DM loss alone closely match their teachers' \textit{CFG-free} samples, indicating that the DM loss is not a mere stabilizer but the key driver that aligns the student to the teacher's CFG-free distribution.}
    \label{fig:cfg_vs_nocfg}
\end{figure}

While continuous scheduling provides flexible on-trajectory supervision, few-step generation inevitably introduces severe numerical truncation errors due to large integration step sizes, causing the inference trajectory to drift off the ideal manifold~\cite{ningelucidating, ning2023input}.
To directly counter this, we propose a novel \textit{Continuous-Time Distribution Matching (CDM) loss}, which intrinsically incorporates a \textit{velocity-driven extrapolation} mechanism into its matching objective.
Instead of restricting supervision to on-trajectory latents, the CDM loss actively probes off-trajectory latents by taking a first-order step along the student's predicted velocity field, and enforces distribution matching upon them.
Acting as a powerful spatial alignment objective, it effectively mitigates off-trajectory drift, empowering the student to self-correct integration errors and recover sharp, high-frequency details.


In summary, to the best of our knowledge, we are the first to migrate the DMD distillation framework from discrete schedules to a continuous optimization space. Our contributions are as follows:
\begin{itemize}
    \item We empirically reveal two key insights in distribution matching: (1) anchoring the training optimization to a fixed set of discrete timesteps is not necessary; and (2) the distribution matching (DM) loss acts not merely as a "regularizer", but drives the student to align with the teacher's CFG-free distribution.
    \item To fully exploit these findings, we propose the CDM framework. This paradigm unifies a \textit{dynamic continuous scheduling} strategy for flexible on-trajectory supervision, and a novel \textit{off-trajectory CDM loss} equipped with velocity-driven extrapolation to actively mitigate numerical integration errors during sampling.
    \item Extensive experimental results demonstrate that our continuous paradigm yields significant performance gains, establishing new state-of-the-art results for few-step image generation across different models (\emph{e.g.,} SD3-Medium and Longcat-Image) without relying on complex auxiliary modules.
\end{itemize}
\section{Related Work}
\label{sec:related}

\paragraph{Diffusion Distillation}
While diffusion models~\cite{ho2020denoising,rombach2022high, songscore} have achieved unprecedented success in visual generation tasks, their iterative sampling process poses a significant computational bottleneck.
To accelerate inference, numerous distillation paradigms have been proposed.
Progressive distillation~\cite{meng2023distillation,sabouralign,salimans2022progressive} accelerates sampling by iteratively training a student to compress two teacher steps into one, progressively halving the required function evaluations.
Consistency models~\cite{kimconsistency, lusimplifying, luo2023latent, peng2025facm, song2023consistency, wang2024phased, zheng2024trajectory} take a different approach by enforcing a self-consistency property: learning a direct mapping from any point along the probability flow ODE trajectory to the trajectory's origin on the data manifold, enabling few-step generation.
Alternatively, adversarial distillation methods~\cite{lin2024sdxl, sauer2024adversarial} leverage a discriminator to align the few-step student's output directly with the real data distribution.
Recent hybrid approaches further combine these paradigms: SANA-Sprint~\cite{chen2025sana-sprint} and SwiftVideo~\cite{sun2026swiftvideo} unify continuous-time consistency distillation with adversarial distribution alignment or trajectory distribution alignment, while TwinFlow~\cite{cheng2025twinflow} pairs consistency modeling with self-adversarial distribution matching to enable high-fidelity one-step generation.

\paragraph{Score-based Distillation}
Score-based distillation originated in text-to-3D generation, where SDS~\cite{pooledreamfusion} and VSD~\cite{wang2023prolificdreamer} leveraged pretrained diffusion scores to optimize 3D representations, establishing the conceptual foundation of distribution matching for distillation.
Extending this paradigm to 2D image generation, Diff-Instruct~\cite{luo2023diff} and DMD~\cite{yin2024one} formulated KL-based distribution matching frameworks for distilling diffusion models into few-step generators, with DMD2~\cite{yin2024improved} further improving stability via adversarial losses.
Subsequent theoretical analyses~\cite{liu2025decoupled, yu2023text} decoupled the score distillation objective, revealing that CFG augmentation drives few-step conversion while the distribution matching term serves as a stabilizing regularizer.
More recently, the DMD framework has been extended along multiple axes: scaling to large flow-based models~\cite{ge2025senseflow}, incorporating RL-based or GAN-based refinement~\cite{chadebec2025flash,jiang2025distribution,ren2024hyper}, combining with consistency distillation or progressive distillation~\cite{fan2025phased,ren2024hyper,wei2026skywork}, introducing scale-wise distillation~\cite{chen2026cross,starodubcev2025scale}, score
identity distillation~\cite{zhou2025few}, or cache-aware distillation~\cite{li20261,nie2026transition}.
Despite these advances, all existing DMD-based methods evaluate the DM loss exclusively at sparse discrete timesteps, leaving the continuous trajectory unoptimized.
To address these limitations, we propose Continuous-Time Distribution Matching (CDM), which introduces a dynamic continuous schedule together with a velocity-driven off-trajectory alignment objective, shifting the optimization to the continuous-time domain.
Notably, a concurrent work~\cite{qin2026soarselfcorrectionoptimalalignment} shares a similar off-trajectory insight with us, but constructs off-trajectory points via re-noising and focuses on post-training alignment rather than distillation.


\section{Method}
\label{sec:method}

We present Continuous-Time Distribution Matching (CDM), a unified distillation framework that lifts the discrete-time DMD paradigm into a fully continuous-time formulation for high-fidelity few-step generation.
We first formalize the decoupled Distribution Matching Distillation (DMD) baseline (\Cref{sec:prelim}).
Building on this, we relax the fixed inference schedule into a dynamic continuous schedule and theoretically examine its implications for distribution matching (\Cref{sec:pilot_study}).
Finally, in \Cref{sec:cdm} we complement these with the CDM loss, which extends supervision from on-trajectory anchors to off-trajectory latents via a velocity-driven extrapolation, regularizing the student's velocity field across the continuous time domain.
The unified training pipeline is illustrated in \Cref{fig:pipeline}.

\subsection{Preliminaries: Decoupled Distribution Matching}
\label{sec:prelim}

The goal of our distillation framework is to train a student flow model $\mathcal{D}_\theta$ capable of generating high-quality samples in $N$ discrete steps, by distilling knowledge from a pre-trained teacher model $\mathcal{D}_\phi$ that typically requires $T \gg N$ steps.
Here, $\mathcal{D}(\mathbf{x}_t, t, \mathbf{c})$ denotes the model prediction that estimates the clean data from the noisy latent $\mathbf{x}_t$ at timestep $t \in (0, 1]$, conditioned on $\mathbf{c}$.
Formally, assuming the underlying neural network $v_\theta$ is trained to predict the velocity field, the clean data estimate $\mathcal{D}_\theta$ is explicitly parameterized as:
\begin{equation}
    \label{eq:data_pred}
    \mathcal{D}_\theta(\mathbf{x}_t, t, \mathbf{c}) = \mathbf{x}_t - t v_\theta(\mathbf{x}_t, t, \mathbf{c}).
\end{equation}

Building upon DMD~\cite{yin2024one}, DMD2~\cite{yin2024improved}, and Decoupled DMD (D-DMD)~\cite{liu2025decoupled}, we employ a backward simulation strategy to construct the sampling trajectory.
Specifically, starting from random noise $\mathbf{x}_{t_1} \sim \mathcal{N}(\mathbf{0}, \mathbf{I})$, we generate the trajectory by numerically integrating the probability flow ODE along the student's predefined discrete time schedule $\{t_1, \ldots, t_N\}$.
During this process, we extract an intermediate latent state $\mathbf{x}_{t_i}$, where the index $i \sim \mathcal{U}\{1, \ldots, N\}$ is uniformly sampled. 
The distillation objective is decoupled into two orthogonal components: a CFG Augmentation (CA) term and a Distribution Matching (DM) term:
\begin{equation}
    \label{eq:dmd_loss}
    \mathcal{L}_{\mathrm{DMD}} = \mathcal{L}_{\mathrm{CA}} + \mathcal{L}_{\mathrm{DM}}.
\end{equation}

\paragraph{CA Loss ($\mathcal{L}_{\mathrm{CA}}$)}
To enforce text-image alignment, the latent $\mathbf{x}_{t_i}$ is passed through the student model to yield the clean data estimate $\mathcal{D}_\theta(\mathbf{x}_{t_i}, t_i, \mathbf{c})$.
This estimate is subsequently perturbed with noise to a random continuous timestep $\tau \in (0, 1]$ to form $\mathbf{z}_{\tau}$.
Following DMD~\cite{yin2024one}, we introduce a dynamic weighting factor $w_\tau = \| \mathcal{D}_\phi(\mathbf{z}_{\tau}, \tau, \mathbf{c}) - \mathcal{D}_\theta(\mathbf{x}_{t_i}, t_i, \mathbf{c}) \|_1^{-1}$ to normalize the gradient's magnitude.
The CA loss is then defined as:
\begin{equation}
    \label{eq:cfg_loss}
    \mathcal{L}_{\mathrm{CA}} = \frac{1}{2}\left\| \mathcal{D}_\theta(\mathbf{x}_{t_i}, t_i, \mathbf{c}) - \operatorname{sg}\left[ \mathcal{D}_\theta(\mathbf{x}_{t_i}, t_i, \mathbf{c}) + \underbrace{ w_\tau \alpha \left( \mathcal{D}_\phi(\mathbf{z}_{\tau}, \tau, \mathbf{c}) - \mathcal{D}_\phi(\mathbf{z}_{\tau}, \tau, \varnothing) \right) }_{\Delta_{\mathrm{ca}}^{\mathrm{real}} \text{ (CFG Augmentation)}} \right] \right\|_2^2,
\end{equation}
where $\mathbf{c}$ is the conditioning text, $\alpha$ is the guidance scale, and $\operatorname{sg}[\cdot]$ is the stop-gradient operator.

\paragraph{DM Loss ($\mathcal{L}_{\mathrm{DM}}$)}
To align the student's marginal distribution with the real data manifold, we similarly reuse the student's clean data estimate $\mathcal{D}_\theta(\mathbf{x}_{t_i}, t_i, \mathbf{c})$.
This estimate is independently perturbed with noise to another random continuous timestep $\tilde{\tau} \in (0, 1]$ to form $\mathbf{z}_{\tilde{\tau}}$.
Using a frozen real teacher $\mathcal{D}_\phi$ and an online-updated fake teacher $\mathcal{D}_\psi$ (which parameterizes the student's score), the DM loss is defined as:
\begin{equation}
    \label{eq:ddm_loss}
    \mathcal{L}_{\mathrm{DM}} = \frac{1}{2}\left\| \mathcal{D}_\theta(\mathbf{x}_{t_i}, t_i, \mathbf{c}) - \operatorname{sg}\left[ \mathcal{D}_\theta(\mathbf{x}_{t_i}, t_i, \mathbf{c}) + \underbrace{ w_{\tilde{\tau}}(\mathcal{D}_\phi(\mathbf{z}_{\tilde{\tau}}, \tilde{\tau}, \mathbf{c}) - \mathcal{D}_\psi(\mathbf{z}_{\tilde{\tau}}, \tilde{\tau}, \mathbf{c}) )}_{\Delta_{\mathrm{dm}}^{\mathrm{real-fake}} \text{ (Distribution Matching)}} \right] \right\|_2^2,
\end{equation}
where $\mathcal{D}_\phi$ and $\mathcal{D}_\psi$ denote the frozen real teacher and the online-updated fake teacher (which parameterizes the student's generated distribution), respectively.

\begin{figure}[t]
    \centering
    \includegraphics[width=\textwidth]{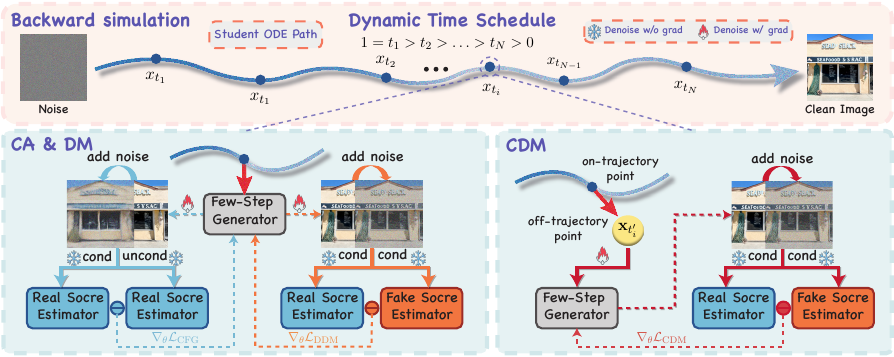}
    \caption{
        \textbf{Overview of Continuous-Time Distribution Matching (CDM).}
        \textbf{Top:} Our approach employs a dynamic continuous time schedule during backward simulation, sampling intermediate anchors uniformly from $(0, 1]$.
        \textbf{Bottom Left:} CFG augmentation (CA) and distribution matching (DM) operate on this dynamic schedule to align text-image conditions and data distributions at on-trajectory anchors.
        \textbf{Bottom Right:} To address inter-anchor inconsistency, the proposed CDM objective explicitly extrapolates off-trajectory latents ($\mathbf{x}_{t_i'}$) using the predicted velocity. 
    }
    \label{fig:pipeline}
\end{figure}

\subsection{Dynamic Time Schedule}
\label{sec:pilot_study}

In vanilla DMD2~\cite{yin2024improved} paradigm, the backward simulation strategy relies on a fixed, predefined set of discrete timesteps matching the target inference schedule, denoted as $\mathcal{S}_{\mathrm{infer}} = \{t_1, \ldots, t_N\}$.
To maintain strict training-inference consistency, prior methods force the backward simulation during training to exclusively operate on these exact points.

However, we propose to break this rigid constraint by introducing a continuous dynamic time schedule.
In each training iteration, the backward simulation length $N$ is no longer fixed but randomly sampled ($N \sim \mathcal{U}\{1, N_{\max}\}$).
We then randomly generate a strictly decreasing continuous time sequence $1 = t_1 > t_2 > \ldots > t_N > 0$, where $1$ represents pure noise and $0$ represents the clean image.
This dynamic schedule brings two independent benefits.
First, the random simulation length $N$ exposes the student to varying numbers of inference steps at training time and lets the teacher provide gradient signals over a more diverse distribution of intermediate latents $\mathbf{x}_{t_i}$.
Second, the student's anchors $t_i$ are no longer confined to the fixed discrete set $\mathcal{S}_{\mathrm{infer}}$; instead, they are drawn from the same continuous domain $(0, 1]$ as the teacher's perturbation timesteps $\tau$ and $\tilde{\tau}$, which remain independently sampled.
This eliminates the mismatch between the discrete student anchors and the continuous teacher supervision in vanilla DMD.

To provide a theoretical motivation for our dynamic time schedule, we examine the optimization from a score-matching perspective by applying Tweedie's formula~\cite{efron2011tweedie} (see \Cref{app:score_matching} for detailed derivations).
Let $p_{\mathrm{real}}(\mathbf{z}_t | \mathbf{c})$ denote the marginal distribution of the real data at a continuous noise level $t$, and $p_{\mathrm{fake}}(\mathbf{z}_t | \mathbf{c})$ represent the fake target distribution.

For the CFG Augmentation (CA) loss, the gradient mathematically defines the direction of an implicit classifier $\log p_{\mathrm{real}}(\mathbf{c}|\mathbf{z}_\tau)$~\cite{yu2023text}, effectively pushing the student's generation toward regions of higher text-image alignment:
\begin{equation}
    \label{eq:dyn_ca_grad}
    \nabla_\theta \mathcal{L}_{\mathrm{CA}} = -w_\tau \alpha \frac{\tau^2}{1-\tau} \left( \frac{\partial \mathcal{D}_\theta(\mathbf{x}_{t_i}, t_i, \mathbf{c})}{\partial \theta} \right)^T \nabla_{\mathbf{z}_\tau} \log p_{\mathrm{real}}(\mathbf{c}|\mathbf{z}_\tau).
\end{equation}

For the Distribution Matching (DM) loss, the formulation reveals its analytical connection to the Kullback-Leibler divergence.
Specifically, optimizing the DM loss corresponds to minimizing the KL divergence $D_{\mathrm{KL}}(p_{\mathrm{gen}}^{\tilde{\tau}} \| p_{\mathrm{real}}^{\tilde{\tau}})$ between the student's generative distribution and the real data distribution at time $\tilde{\tau}$:
\begin{equation}
    \label{eq:dyn_dm_grad}
    \nabla_\theta \mathcal{L}_{\mathrm{DM}} = -w_{\tilde{\tau}} \frac{\tilde{\tau}^2}{1-\tilde{\tau}} \left( \frac{\partial \mathcal{D}_\theta(\mathbf{x}_{t_i}, t_i, \mathbf{c})}{\partial \theta} \right)^T \left( \nabla_{\mathbf{z}_{\tilde{\tau}}} \log p_{\mathrm{real}}(\mathbf{z}_{\tilde{\tau}}|\mathbf{c}) - \nabla_{\mathbf{z}_{\tilde{\tau}}} \log p_{\mathrm{fake}}(\mathbf{z}_{\tilde{\tau}}|\mathbf{c}) \right).
\end{equation}

Crucially, the student's input timestep $t_i$ and the teacher's perturbation timesteps $\tau, \tilde{\tau}$ are independently sampled from the same continuous distribution over $(0, 1]$.
In expectation, this mechanism encourages both the CA and DM gradients in \Cref{eq:dyn_ca_grad,eq:dyn_dm_grad} to regularize the student's velocity field across the continuous time domain, rather than overfitting to sparse discrete anchors. 
While this continuous formulation provides a theoretical intuition for a smoother velocity field, we empirically validate its generalization benefits in \Cref{fig:schedule} and our experiments (\Cref{sec:experiments}).

\subsection{Continuous-Time Distribution Matching (CDM)}
\label{sec:cdm}

The dynamic continuous schedule introduced in \Cref{sec:pilot_study} provides supervision at randomly sampled anchors visited by backward simulation and can in principle cover any point at $(0, 1]$ given enough iterations. The supervision is applied to one anchor at a time: at each $t_i$, the loss only constrains the student's prediction $\mathcal{D}_\theta(\mathbf{x}_{t_i}, t_i, \mathbf{c})$ to match the target distribution at that single point.
It does not constrain the student's velocity $v_\theta$ to remain consistent across adjacent time steps.
Few-step inference, however, depends on this property: each Euler step from $t_j$ to $t_{j-1}$ introduces an error of order $\mathcal{O}((\Delta t)^2 \sup_\tau \|dv_\theta/d\tau\|)$, where the last term measures how rapidly $v_\theta$ changes between adjacent time steps (see \Cref{app:euler_error} for a detailed derivation).
Supervising each anchor in isolation gives no direct control over this term.

To reduce this inter-anchor inconsistency, we introduce the CDM loss, which adds supervision on intermediate latents between adjacent anchors.
Given an on-trajectory latent $\mathbf{x}_{t_i}$ and its predicted velocity $v_{t_i} = v_\theta(\mathbf{x}_{t_i}, t_i, \mathbf{c})$, we sample a paired anchor $t_i' \sim \mathcal{U}(0, 1]$ independent of the integration schedule and perform a first-order Euler extrapolation:
\begin{equation}
    \label{eq:cdm_extrapolation}
    \mathbf{x}_{t_i'} = \mathbf{x}_{t_i} + (t_i' - t_i)\,v_{t_i}.
\end{equation}
Because the underlying probability flow ODE trajectory is curved, a large stride $|t_i' - t_i|$ along the linearized velocity $v_{t_i}$ produces an intermediate latent $\mathbf{x}_{t_i'}$ that lies between (or beyond) the discrete anchors and is not visited by standard backward simulation.

To supervise $\mathbf{x}_{t_i'}$, we construct the target latent directly from the local clean data estimate predicted at the extrapolated point.
Specifically, we pass $\mathbf{x}_{t_i'}$ through the student model to obtain the local prediction $\hat{\mathbf{x}}_0^{(i')} = \mathcal{D}_\theta(\mathbf{x}_{t_i'}, t_i', \mathbf{c})$, and re-noise it to a continuous time $\hat{\tau} \sim \mathcal{U}(0, 1]$:
\begin{equation}
    \label{eq:cdm_target_state}
    \mathbf{z}_{\hat{\tau}} = (1-\hat{\tau})\operatorname{sg}[\hat{\mathbf{x}}_0^{(i')}] + \hat{\tau}\boldsymbol{\epsilon}_{\hat{\tau}}, \quad \boldsymbol{\epsilon}_{\hat{\tau}} \sim \mathcal{N}(\mathbf{0}, \mathbf{I}).
\end{equation}
By anchoring the reference target to the local estimate $\hat{\mathbf{x}}_0^{(i')}$, we establish a self-consistency constraint for the student's vector field. Due to the Euler extrapolation, $\mathbf{x}_{t_i'}$ naturally drifts off the ideal sampling trajectory. Re-noising this drifted prediction yielding $\mathbf{z}_{\hat{\tau}}$ allows the frozen teacher to evaluate the local score matching error. This localized supervision essentially penalizes invalid velocity predictions outside the main trajectory, promoting a smoother and more regularized flow for few-step integration.

The CDM loss is then defined on the extrapolated input and the $\hat{\mathbf{x}}_0^{(i')}$-anchored target:
\begin{equation}
    \label{eq:cdm_base_loss}
    \mathcal{L}_{\mathrm{CDM}} = \frac{1}{2}\left\| \mathcal{D}_\theta(\mathbf{x}_{t_i'}, t_i', \mathbf{c}) - \operatorname{sg}\Bigl[ \mathcal{D}_\theta(\mathbf{x}_{t_i'}, t_i', \mathbf{c}) + w_{\hat{\tau}} \underbrace{ \left( \mathcal{D}_\phi(\mathbf{z}_{\hat{\tau}}, \hat{\tau}, \mathbf{c}) - \mathcal{D}_\psi(\mathbf{z}_{\hat{\tau}}, \hat{\tau}, \mathbf{c}) \right) }_{\Delta_{\mathrm{cdm}}^{\mathrm{real-fake}}} \Bigr] \right\|_2^2.
\end{equation}
By matching the student's prediction at the off-trajectory latent $\mathbf{x}_{t_i'}$ to the target distribution, $\mathcal{L}_{\mathrm{CDM}}$ constrains $v_\theta$ across the continuous interval, reducing the inter-anchor inconsistency.

\paragraph{Full Training Objective}
Our comprehensive training objective unifies these mathematical components into a single sum:
\begin{equation}
    \label{eq:full_objective}
    \mathcal{L} = \mathcal{L}_{\mathrm{CA}} + \mathcal{L}_{\mathrm{DM}} + \mathcal{L}_{\mathrm{CDM}}.
\end{equation}

\section{Experiments}
\label{sec:experiments}

\subsection{Experimental Setup}
\label{sec:exp_setup}

\paragraph{Experiment Setting}
We conduct our main experiments on SD3-Medium~\cite{esser2024scaling} at a resolution of $1024 \times 1024$.
For evaluation, we employ Aesthetic Score (AES)~\cite{schuhmann2022laionaesthetics}, PickScore~\cite{kirstain2023pick}, HPS v3~\cite{ma2025hpsv3}, and CLIP Score (ViT-H-14)~\cite{hessel2021clipscore} on 2K prompts sampled from the test split of the PickScore dataset~\cite{kirstain2023pick}.
We additionally report fine-grained prompt adherence on DPG-Bench (DPG)~\cite{hu2024ella} using 1K prompts.
For a comprehensive evaluation, we compare our method against several leading few-step generation baselines, including Hyper-SD~\cite{ren2024hyper}, Flash~\cite{chadebec2025flash}, TDM~\cite{luo2025learning}, DMD2~\cite{yin2024improved}, and D-DMD~\cite{liu2025decoupled}.
Furthermore, to demonstrate the broad applicability of our approach, we also extend our experiments to Longcat-Image~\cite{team2025longcat}.
Detailed experiment configurations are provided in \Cref{app:appendix_exp_details}.

\subsection{Main Results}
\label{sec:main_results}

\begin{figure}[ht]
    \centering
    \includegraphics[width=\textwidth]{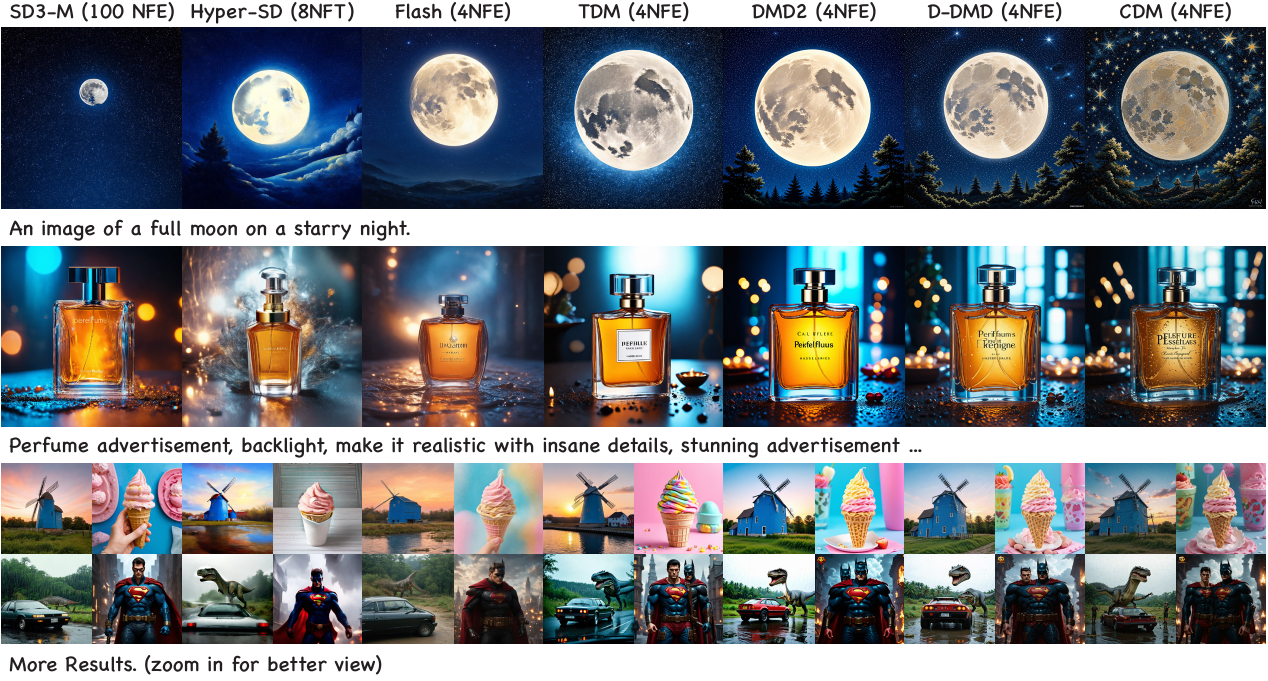}
    \caption{\textbf{Qualitative comparison on SD3-Medium.} CDM produces more photorealistic results with richer details than competing methods. All results are generated using the same initial noise and random seed for fair comparison.}
    \label{fig:main_results}
\end{figure}

\begin{table*}[t]
\centering
\caption{\textbf{Quantitative comparison of different methods on SD3-Medium and Longcat-Image.}
For each backbone, the best and {second-best} results are highlighted in \textbf{bold} and \underline{underline}, respectively; the base model serves as a reference and is excluded from the ranking.
Methods marked with $*$ denote our reproduced results.
Image-Free indicates methods that do not rely on real images during distillation; Continuous indicates methods whose supervision is applied at arbitrary continuous timesteps.}

\setlength{\tabcolsep}{3.5pt} 

\resizebox{\textwidth}{!}{
\begin{tabular}{lc ccccc cc}
\toprule
\textbf{Method} & \textbf{NFE} & \textbf{Aesthetic}$\uparrow$ & \textbf{DPGBench}$\uparrow$ & \textbf{PickScore}$\uparrow$ & \textbf{HPSv3}$\uparrow$ & \textbf{CLIPScore}$\uparrow$ & \textbf{Image-Free} & \textbf{Continuous} \\
\midrule
\multicolumn{9}{l}{\textit{\color{gray}SD3-Medium}} \\
\rowcolor[HTML]{E6E6E6} 
Base~\cite{esser2024scaling} & 100 & 5.885 & 85.04 & 21.73 & 8.189 & 28.60 & - & - \\
Hyper-SD~\cite{ren2024hyper} & 8 & 5.180 & 80.43 & 20.82 & 6.054 & 27.93 & {\color{red}\xmark} & {\color{green}\cmark} \\
Flash~\cite{chadebec2025flash} & 4 & 5.968 & 80.47 & 21.69 & 8.282 & \textbf{28.18} & {\color{red}\xmark} & {\color{red}\xmark} \\
TDM~\cite{luo2025learning} & 4 & 6.013 & 83.12 & 21.61 & 8.468 & 27.63 & {\color{green}\cmark} & {\color{red}\xmark} \\
DMD2*~\cite{yin2024improved} & 4 & \underline{6.038} & 83.96 & 21.58 & 8.419 & 27.56 & {\color{green}\cmark} & {\color{red}\xmark} \\
D-DMD*~\cite{liu2025decoupled} & 4 & \underline{6.038} & \underline{84.52} & \underline{21.85} & \underline{9.176} & 27.69 & {\color{green}\cmark} & {\color{red}\xmark} \\
\rowcolor[HTML]{EBE4F2}
CDM (Ours) & 4 & \textbf{6.075} & \textbf{85.26} & \textbf{21.95} & \textbf{9.561} & \underline{27.98} & {\color{green}\cmark} & {\color{green}\cmark} \\
\midrule
\multicolumn{9}{l}{\textit{\color{gray}Longcat-Image}} \\
\rowcolor[HTML]{E6E6E6} 
Base~\cite{team2025longcat} & 100 & 5.926 & 87.08 & 21.65 & 9.450 & 26.78 & - & - \\
DMD2*~\cite{yin2024improved} & 4 & \underline{5.800} & 87.12 & 21.07 & 8.803 & \textbf{26.99} & {\color{green}\cmark} & {\color{red}\xmark} \\
D-DMD*~\cite{liu2025decoupled} & 4 & 5.782 & \underline{88.04} & \underline{21.23} & \underline{9.629} & 26.57 & {\color{green}\cmark} & {\color{red}\xmark} \\
\rowcolor[HTML]{EBE4F2} 
CDM (Ours) & 4 & \textbf{5.919} & \textbf{88.35} & \textbf{21.53} & \textbf{10.65} & \underline{26.72} & {\color{green}\cmark} & {\color{green}\cmark} \\
\bottomrule
\end{tabular}
} 
\label{tab:metrics_comparison}
\end{table*}

\paragraph{Quantitative Results}
As shown in \Cref{tab:metrics_comparison}, CDM achieves competitive performance against existing few-step baselines on both SD3-Medium and Longcat-Image with only 4 NFE.
On SD3-Medium, CDM obtains the best Aesthetic ($6.075$), DPGBench ($85.26$), PickScore ($21.95$), and HPSv3 ($9.561$), while maintaining a highly competitive CLIPScore.
Among image-free methods, CDM compares favorably with D-DMD~\cite{liu2025decoupled}, achieving consistent improvements across all metrics, notably improving HPSv3 from $9.176$ to $9.561$.
A similar trend holds on Longcat-Image, where CDM attains the best results on Aesthetic, DPGBench, PickScore, and HPSv3.
Interestingly, our 4-NFE student matches or even surpasses the 100-NFE pretrained teacher on a range of metrics (\emph{e.g.}, DPGBench and HPSv3) on both backbones, suggesting that the proposed continuous-time optimization framework provides supervision signals that go beyond merely replicating the teacher's outputs.
More quantitative comparisons, including training and inference efficiency, are provided in \Cref{app:appendix_efficiency}.

\paragraph{Qualitative Comparison}
\Cref{fig:main_results} presents a side-by-side qualitative comparison against representative few-step baselines as well as the 100-NFE teacher.
Across diverse prompts, CDM consistently yields sharper textures and fine-grained details (e.g., background elements and material reflections), and stronger semantic adherence to multi-entity compositional prompts, while competing baselines often exhibit blurry high-frequency content or missing attributes.
Notably, despite operating with only 4 NFE, CDM matches or even visually surpasses the 100-NFE teacher in perceptual sharpness and aesthetics on many cases, corroborating the quantitative trends in \Cref{tab:metrics_comparison}.
Additional visual results on both SD3-Medium and Longcat-Image are provided in \Cref{sec:appendix_more_qualitative}.
\subsection{Ablation Study}
\label{sec:ablation}

\begin{figure}[t]
    \centering
    \includegraphics[width=\textwidth]{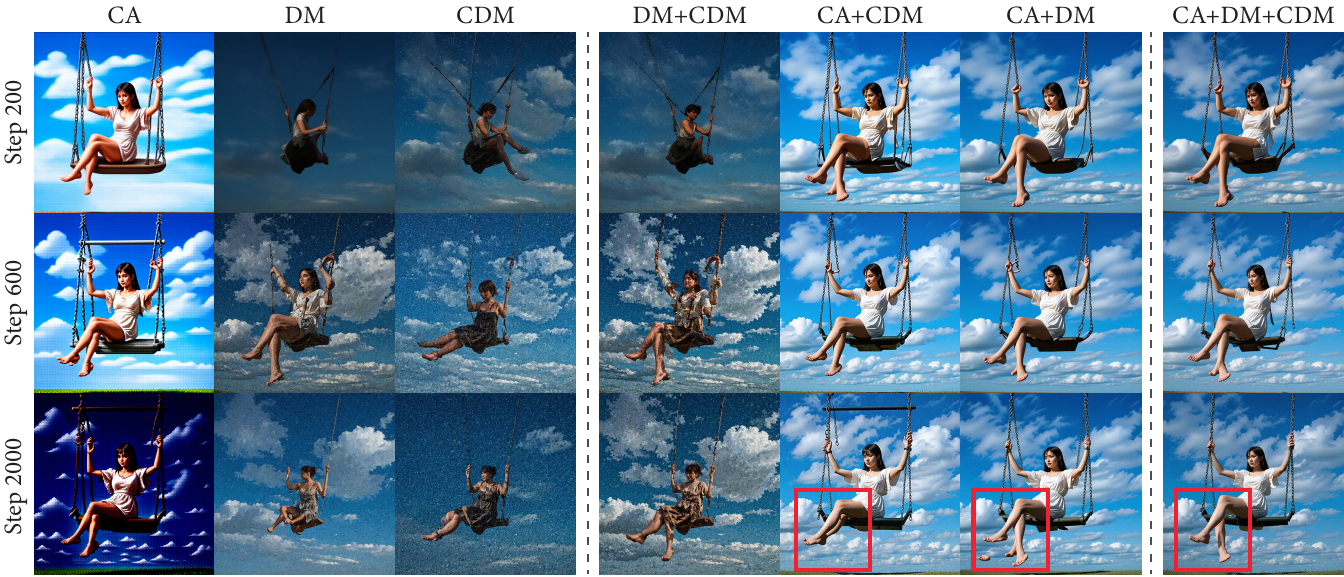}
    \caption{
        \textbf{Qualitative ablation of loss components across training steps.}
        \textbf{Left:} Individual losses (CA, DM, CDM) in isolation.
        \textbf{Right:} Pairwise and full combinations. 
        Partial combinations suffer from brightness collapse or degraded local fidelity at later stages, whereas our full objective (CA+DM+CDM) effectively preserves both global semantic coherence and local details.
    }
    \label{fig:ablation}
\end{figure}

\begin{table}[t]
\centering
\caption{
    \textbf{Ablation study on SD3-Medium at 4 NFE.} 
    \textbf{Left:} Effect of individual loss components. 
    \textbf{Right:} Analysis of core mechanism designs (time schedule, perturbation strategy, and target latent construction). The full CDM design achieves the best performance balance.
}
\label{tab:ablation}
\vspace{2mm}

\begin{minipage}[t]{0.45\textwidth}
\centering
\resizebox{\linewidth}{!}{
\begin{tabular}{l|ccccc}
\toprule
\textbf{Configuration} & \textbf{AES$\uparrow$} & \textbf{DPG$\uparrow$} & \textbf{PICK$\uparrow$} & \textbf{HPSv3$\uparrow$} & \textbf{CLIP$\uparrow$} \\
\midrule
\multicolumn{6}{l}{\textit{(a) Single Loss Ablation}} \\
\hspace{3mm} w/o $\mathcal{L}_{\mathrm{CA}}$ & 5.861 & 72.87 & 21.05 & 8.128 & 24.78 \\
\hspace{3mm} w/o $\mathcal{L}_{\mathrm{DM}}$ & 6.016 & 84.57 & 21.75 & 8.954 & 27.66 \\
\hspace{3mm} w/o $\mathcal{L}_{\mathrm{CDM}}$ & 6.067 & 85.12 & 21.85 & 9.153 & 27.91 \\
\midrule
\multicolumn{6}{l}{\textit{(b) Dual Loss Ablation}} \\
\hspace{3mm} w/o $\mathcal{L}_{\mathrm{DM}} \& \mathcal{L}_{\mathrm{CDM}}$ & 4.634 & 3.45 & 17.50 & -10.15 & 14.60 \\
\hspace{3mm} w/o $\mathcal{L}_{\mathrm{CA}} \& \mathcal{L}_{\mathrm{CDM}}$ & 5.787 & 70.60 & 20.82 & 7.258 & 25.31 \\
\hspace{3mm} w/o $\mathcal{L}_{\mathrm{CA}} \& \mathcal{L}_{\mathrm{DM}}$ & 5.778 & 72.38 & 20.80 & 7.331 & 24.78 \\
\midrule
\rowcolor[HTML]{EBE4F2}
Full CDM & \textbf{6.075} & \textbf{85.26} & \textbf{21.95} & \textbf{9.561} & \textbf{27.98} \\
\bottomrule
\end{tabular}
}
\end{minipage}\hfill
\begin{minipage}[t]{0.53\textwidth}
\centering
\resizebox{\linewidth}{!}{
\begin{tabular}{l|ccccc}
\toprule
\textbf{Model Variant} & \textbf{AES$\uparrow$} & \textbf{DPG$\uparrow$} & \textbf{PICK$\uparrow$} & \textbf{HPSv3$\uparrow$} & \textbf{CLIP$\uparrow$} \\
\midrule
\multicolumn{6}{l}{\textit{(a) Time Schedule}} \\
\hspace{3mm} w/ Fixed Schedule & 6.051 & 83.84 & 21.89 & 9.482 & 27.75 \\
\midrule
\multicolumn{6}{l}{\textit{(b) Off-trajectory Perturbation}} \\
\hspace{3mm} w/o Perturbation (on-traj) & 6.027 & 84.43 & 21.94 & 9.374 & 27.90 \\
\hspace{3mm} w/ Gaussian Perturb. & 6.040 & 84.65 & 21.92 & 9.516 & 27.88 \\
\midrule
\multicolumn{6}{l}{\textit{(c) Target Latent Construction}} \\
\hspace{3mm} w/ Full-trajectory $\hat{\mathbf{x}}_0$ target & 6.026 & 85.14 & 21.92 & 9.346 & 27.97 \\
\midrule
\rowcolor[HTML]{EBE4F2}
Full CDM & \textbf{6.075} & \textbf{85.26} & \textbf{21.95} & \textbf{9.561} & \textbf{27.98} \\
\bottomrule
\end{tabular}
}
\end{minipage}

\end{table}

\paragraph{Loss Components}
We ablate each loss component in \Cref{tab:ablation} and \Cref{fig:ablation}.
Relying solely on $\mathcal{L}_{\mathrm{CA}}$ leads to structural collapse, whereas using only $\mathcal{L}_{\mathrm{DM}}$ or $\mathcal{L}_{\mathrm{CDM}}$ recovers visual quality but struggles with prompt adherence (\emph{e.g.}, noticeably lower CLIP scores). 
Pairing $\mathcal{L}_{\mathrm{CA}}$ with either distribution-matching loss bridges this gap, significantly improving both alignment and aesthetics. 
Ultimately, our full objective ($\mathcal{L}_{\mathrm{CA}}+\mathcal{L}_{\mathrm{DM}}+\mathcal{L}_{\mathrm{CDM}}$) achieves the best performance across all metrics, with HPSv3 peaking at 9.561. 
This confirms their complementary roles: $\mathcal{L}_{\mathrm{CA}}$ anchors structure and semantic alignment, while $\mathcal{L}_{\mathrm{DM}}$ and $\mathcal{L}_{\mathrm{CDM}}$ provide essential on- and off-trajectory distributional supervision, respectively.

\paragraph{Core Mechanism Design}
To validate the critical design choices in Continuous-Time Distribution Matching, we perform an in-depth ablation on its three core mathematical components in the right panel of \Cref{tab:ablation}: the dynamic time schedule, the off-trajectory perturbation strategy, and the target latent construction.
First, reverting our \emph{dynamic schedule} to a standard fixed schedule leads to a considerable drop in generation fidelity, confirming that unifying the simulated trajectory and continuous distribution matching limits structural discrepancy.
Second, replacing our \emph{velocity-driven extrapolation} with a Gaussian noise baseline---implemented by first predicting the clean data $\mathbf{x}_0$ from $\mathbf{x}_t$ and re-adding noise to yield $\mathbf{x}'_t$---or removing perturbations entirely, degrades overall performance.
This indicates that conventional re-noising fails to capture meaningful off-trajectory states, whereas our velocity-based extrapolation accurately simulates the exact truncation drift encountered during large-step inference.
Finally, utilizing the full-trajectory \emph{final generation} $\hat{\mathbf{x}}_0$ rather than the intermediate target $\hat{\mathbf{x}}_0^{(i')}$ as the teacher's input yields suboptimal results. This validates that anchoring the supervision to localized predictions provides a more direct and effective signal for error correction than relying on the extended backward simulation.

\section{Conclusion}
\label{sect:conclusion}
In this paper, we present Continuous-Time Distribution Matching (CDM), a novel distillation framework for high-quality few-step diffusion generation.
Existing discrete-time methods optimize on fixed timesteps, leading to accumulated discretization errors and detail degradation during few-step inference.
To bridge this gap, CDM shifts the optimization into a continuous-time space, leveraging dynamic scheduling and an off-trajectory alignment objective ($\mathcal{L}_{\mathrm{CDM}}$) to explicitly simulate and correct truncation drifts back to the target data manifold.
Extensive experiments on SD3-Medium and Longcat-Image demonstrate that CDM effectively recovers sharp textures and semantic adherence, achieving state-of-the-art 4-step generation. 
Notably, it accomplishes this purely through robust continuous supervision, bypassing the need for adversarial training or costly external reward models.
We hope this work paves the way for more accessible diffusion distillation and inspires future extensions to complex visual synthesis.

\bibliographystyle{plain}
\bibliography{references}

\newpage
\appendix
\crefalias{section}{appendix}

\section{Limitations}
\label{app:limitations}

While CDM achieves strong few-step generation quality, it has several limitations that we leave for future work.
First, although our dynamic continuous schedule and CDM loss do not introduce any additional cost at inference time, they do increase per-iteration training cost: the dynamic schedule samples a variable simulation length $N \sim \mathcal{U}\{1, N_{\max}\}$ that prolongs the average backward simulation, and the CDM loss requires an extra forward pass of the real and fake teachers on the extrapolated off-trajectory latent $\mathbf{x}_{t_i'}$ (see \Cref{app:appendix_efficiency} for detailed results).
Second, as a distillation framework, CDM is fundamentally upper-bounded by the teacher: the DM and CDM losses both rely on the teacher's score as the supervision signal, so concepts or compositions that the teacher itself handles poorly are unlikely to be recovered through distillation alone, as suggested by the CFG-free analysis in \Cref{fig:cfg_vs_nocfg}.
Third, our empirical study is restricted to text-to-image generation such as SD3-Medium and Longcat-Image; we will explore extending CDM to text-and-image-to-image (TI2I) editing and to video diffusion models, where trajectory length and temporal consistency play a larger role, in future studies.

\section{Broader Impact}
\label{app:broader_impact}

CDM offers a more efficient distillation recipe for diffusion models, reducing the inference cost of high-quality text-to-image generation by an order of magnitude and thereby improving the accessibility of these models on commodity hardware.
Since our work only distills a pre-trained teacher and does not introduce new generative capabilities or training data, its potential risks---such as the misuse of generated imagery for misinformation or copyright infringement---are largely inherited from the underlying teacher model rather than amplified by our method.
We encourage practitioners deploying CDM-distilled models to combine them with established safeguards such as NSFW filtering, invisible watermarking, and content provenance standards (e.g., C2PA).

\section{Experiment Details}
\label{app:appendix_exp_details}

\subsection{Training Hyperparameters}

\subsubsection{SD3-Medium}

We apply our framework to distill the pre-trained SD3-Medium~\cite{esser2024scaling} into a 4-step student model.
The optimization relies on a unified objective comprising CFG Augmentation (CA), Distribution Matching (DM), and Continuous-Time Distribution Matching (CDM) regularization, combined with equal weights.
For the CA loss, we adopt the same teacher timestep sample strategy as in D-DMD~\cite{liu2025decoupled}.
The training dataset consists of 200K prompts randomly sampled from the training sets of T2I-2M~\cite{zou2024text2image}, LAION~\cite{schuhmann2022laion}, ShareGPT-4o-Image~\cite{chen2025sharegpt4oimg}, PickScore~\cite{kirstain2023pick}, and OCR~\cite{liu2025flow}.
We perform full fine-tuning on the student network using the AdamW optimizer with a batch size of 128.
The learning rate is set to $1 \times 10^{-5}$ for the student generator and $5 \times 10^{-6}$ for the fake teacher.
The weight decay is set to 0.001, and $\beta$ values are $(0.9, 0.999)$.
Following the Two Time-scale Update Rule (TTUR), the fake teacher is updated 2 times per student generator update.
The CFG guidance scale ($\alpha$) is maintained at 7.0, and the dynamic schedule length ($N_{\max}$) is set to 28.
The entire distillation process runs for 4K iterations on 16 A100 GPUs, which takes approximately 24 hours.

\subsubsection{LongCat}

The distillation of our LongCat~\cite{team2025longcat} 4-step generator shares the same basic training settings as SD3-Medium.
To handle its large model size efficiently, we employ LoRA~\cite{hu2022lora} fine-tuning with rank 64 and alpha 128.
The training utilizes a batch size of 64 and requires 2K iterations on 16 A100 GPUs, taking approximately 24 hours to achieve optimal convergence.

\subsection{Baseline Implementation Details}
\label{sec:appendix_baseline_details}

All baselines use SD3-Medium~\cite{esser2024scaling} as the teacher backbone, with the number of function evaluations (NFE) and inference-time hyperparameters kept consistent with each method's official configuration.
For TDM~\cite{luo2025learning} (\url{https://github.com/Luo-Yihong/TDM}), Hyper-SD~\cite{ren2024hyper} (\url{https://huggingface.co/ByteDance/Hyper-SD}), and Flash-Diffusion~\cite{chadebec2025flash} (\url{https://github.com/gojasper/flash-diffusion}), we directly use their official checkpoints and recommended generation settings.
For DMD2~\cite{yin2024improved} and D-DMD~\cite{liu2025decoupled}, since no official SD3-Medium checkpoints are publicly available, we re-implement them within our unified framework under the exact same setting as our CDM and notably without the GAN-based adversarial loss.

\section{\texorpdfstring{A Score-Matching Perspective on $\mathcal{L}_{\mathrm{CA}}$ and $\mathcal{L}_{\mathrm{DM}}$}{A Score-Matching Perspective on L\_CA and L\_DM}}
\label{app:score_matching}

Tweedie's formula~\cite{efron2011tweedie} provides the bridge between a denoiser's posterior-mean prediction and the underlying score function.
We first derive its form under the flow-matching interpolation used throughout this paper (\Cref{app:tweedie_subsec}), and then apply it to interpret the gradients of $\mathcal{L}_{\mathrm{CA}}$ and $\mathcal{L}_{\mathrm{DM}}$ in \Cref{eq:cfg_loss,eq:ddm_loss} (\Cref{app:ca_grad_subsec,app:dm_grad_subsec}).
The two derivations together justify the claim in \Cref{sec:pilot_study} that, under the dynamic continuous schedule, both losses regularize the student's velocity field uniformly over $(0, 1]$.

Throughout this section, $p_{\mathrm{real}}(\mathbf{z}_\tau | \mathbf{c})$ denotes the marginal distribution of the real data (modeled by the frozen teacher $\mathcal{D}_\phi$) at continuous noise level $\tau$, and $p_{\mathrm{fake}}(\mathbf{z}_\tau | \mathbf{c})$ denotes the corresponding distribution implicitly defined by the online-updated fake teacher $\mathcal{D}_\psi$.

\subsection{Tweedie's Formula under the Flow-Matching Interpolation}
\label{app:tweedie_subsec}

For completeness, we derive Tweedie's formula under the flow-matching interpolation.

\begin{proposition}[Tweedie's Formula for Flow Matching]
\label{prop:tweedie}
\label{app:tweedie}
Consider the flow-matching forward process that interpolates between clean data $\mathbf{x}_0 \sim p_{\mathrm{data}}$ and Gaussian noise $\boldsymbol{\epsilon} \sim \mathcal{N}(\mathbf{0}, \mathbf{I})$:
\begin{equation}
    \label{eq:fm_forward}
    \mathbf{z}_\tau = (1 - \tau)\mathbf{x}_0 + \tau \boldsymbol{\epsilon}, \quad \tau \in [0, 1].
\end{equation}
Then the posterior mean of the clean data given the noisy observation satisfies
\begin{equation}
    \label{eq:tweedie_fm}
    \mathbb{E}[\mathbf{x}_0 | \mathbf{z}_\tau]
    = \frac{\mathbf{z}_\tau + \tau^2 \nabla_{\mathbf{z}_\tau} \log p(\mathbf{z}_\tau)}{1 - \tau}.
\end{equation}
\end{proposition}

A denoiser $\mathcal{D}(\mathbf{z}_\tau, \tau)$ trained under the mean-squared-error objective converges to the posterior mean $\mathbb{E}[\mathbf{x}_0 | \mathbf{z}_\tau]$, so \Cref{eq:tweedie_fm} also yields the score-prediction identity used in the main text.

\begin{proof}
Given a clean sample $\mathbf{x}_0$, the conditional distribution of $\mathbf{z}_\tau$ is Gaussian:
\begin{equation}
    p(\mathbf{z}_\tau | \mathbf{x}_0)
    = \mathcal{N}\bigl(\mathbf{z}_\tau;\, (1-\tau)\mathbf{x}_0,\, \tau^2 \mathbf{I}\bigr).
\end{equation}

The marginal distribution of $\mathbf{z}_\tau$ is obtained by integrating over the data distribution:
\begin{equation}
    p(\mathbf{z}_\tau)
    = \int p(\mathbf{x}_0)\,
      \mathcal{N}\bigl(\mathbf{z}_\tau;\, (1-\tau)\mathbf{x}_0,\, \tau^2 \mathbf{I}\bigr)\,
      d\mathbf{x}_0.
\end{equation}

Differentiating $\log p(\mathbf{z}_\tau)$ with respect to $\mathbf{z}_\tau$ gives
\begin{align}
    \nabla_{\mathbf{z}_\tau} \log p(\mathbf{z}_\tau)
    &= \frac{\nabla_{\mathbf{z}_\tau} p(\mathbf{z}_\tau)}{p(\mathbf{z}_\tau)} \nonumber \\
    &= \frac{1}{p(\mathbf{z}_\tau)} \int p(\mathbf{x}_0)\,
      \nabla_{\mathbf{z}_\tau} \mathcal{N}\bigl(\mathbf{z}_\tau;\, (1-\tau)\mathbf{x}_0,\, \tau^2 \mathbf{I}\bigr)\,
      d\mathbf{x}_0.
\end{align}
Using the identity for the derivative of a Gaussian density,
\begin{equation*}
    \nabla_{\mathbf{z}_\tau} \mathcal{N}(\mathbf{z}_\tau;\, \boldsymbol{\mu},\, \tau^2\mathbf{I})
    = \frac{\boldsymbol{\mu} - \mathbf{z}_\tau}{\tau^2}\,
      \mathcal{N}(\mathbf{z}_\tau;\, \boldsymbol{\mu},\, \tau^2\mathbf{I}),
\end{equation*}
with $\boldsymbol{\mu} = (1-\tau)\mathbf{x}_0$, we obtain
\begin{align}
    \nabla_{\mathbf{z}_\tau} \log p(\mathbf{z}_\tau)
    &= \frac{1}{p(\mathbf{z}_\tau)}
       \int p(\mathbf{x}_0)\,
       \frac{(1-\tau)\mathbf{x}_0 - \mathbf{z}_\tau}{\tau^2}\,
       \mathcal{N}\bigl(\mathbf{z}_\tau;\, (1-\tau)\mathbf{x}_0,\, \tau^2 \mathbf{I}\bigr)\,
       d\mathbf{x}_0 \nonumber \\
    &= \frac{(1-\tau)\mathbb{E}[\mathbf{x}_0 | \mathbf{z}_\tau] - \mathbf{z}_\tau}{\tau^2},
\end{align}
where the second equality follows from the definition of the posterior mean,
\begin{equation*}
    \mathbb{E}[\mathbf{x}_0 | \mathbf{z}_\tau]
    = \int \mathbf{x}_0\, p(\mathbf{x}_0 | \mathbf{z}_\tau)\, d\mathbf{x}_0.
\end{equation*}

Solving for $\mathbb{E}[\mathbf{x}_0 | \mathbf{z}_\tau]$ yields
\begin{equation}
    (1-\tau)\,\mathbb{E}[\mathbf{x}_0 | \mathbf{z}_\tau]
    = \mathbf{z}_\tau + \tau^2 \nabla_{\mathbf{z}_\tau} \log p(\mathbf{z}_\tau),
\end{equation}
from which \Cref{eq:tweedie_fm} follows.
\end{proof}

\subsection{Score-Matching View of the CA Gradient}
\label{app:ca_grad_subsec}

Treating the stop-gradient target in \Cref{eq:cfg_loss} as a constant, the chain rule gives
\begin{equation}
    \label{eq:app_cfg_grad}
    \nabla_\theta \mathcal{L}_{\mathrm{CA}} = -w_\tau \alpha \left( \frac{\partial \mathcal{D}_\theta(\mathbf{x}_{t_i}, t_i, \mathbf{c})}{\partial \theta} \right)^{\!\top} \!\left( \mathcal{D}_\phi(\mathbf{z}_{\tau}, \tau, \mathbf{c}) - \mathcal{D}_\phi(\mathbf{z}_{\tau}, \tau, \varnothing) \right).
\end{equation}
Applying \Cref{prop:tweedie} to both teacher predictions and using $p_{\mathrm{real}}(\mathbf{z}_\tau | \varnothing) = p_{\mathrm{real}}(\mathbf{z}_\tau)$, the prediction difference reduces to
\begin{equation}
    \label{eq:app_cfg_score}
    \mathcal{D}_\phi(\mathbf{z}_{\tau}, \tau, \mathbf{c}) - \mathcal{D}_\phi(\mathbf{z}_{\tau}, \tau, \varnothing) = \frac{\tau^2}{1-\tau} \left( \nabla_{\mathbf{z}_\tau} \log p_{\mathrm{real}}(\mathbf{z}_\tau|\mathbf{c}) - \nabla_{\mathbf{z}_\tau} \log p_{\mathrm{real}}(\mathbf{z}_\tau) \right).
\end{equation}
Bayes' rule $\log p(\mathbf{c}|\mathbf{z}_\tau) = \log p(\mathbf{z}_\tau|\mathbf{c}) - \log p(\mathbf{z}_\tau) + \log p(\mathbf{c})$, differentiated with respect to $\mathbf{z}_\tau$, eliminates the data-independent prior and gives the implicit-classifier identity
\begin{equation}
    \label{eq:app_cfg_bayes}
    \nabla_{\mathbf{z}_\tau} \log p_{\mathrm{real}}(\mathbf{z}_\tau|\mathbf{c}) - \nabla_{\mathbf{z}_\tau} \log p_{\mathrm{real}}(\mathbf{z}_\tau) = \nabla_{\mathbf{z}_\tau} \log p_{\mathrm{real}}(\mathbf{c}|\mathbf{z}_\tau).
\end{equation}
Substituting back into \Cref{eq:app_cfg_grad},
\begin{equation}
    \label{eq:app_cfg_optim}
    \nabla_\theta \mathcal{L}_{\mathrm{CA}} = -w_\tau \alpha \frac{\tau^2}{1-\tau} \left( \frac{\partial \mathcal{D}_\theta(\mathbf{x}_{t_i}, t_i, \mathbf{c})}{\partial \theta} \right)^{\!\top} \nabla_{\mathbf{z}_\tau} \log p_{\mathrm{real}}(\mathbf{c}|\mathbf{z}_\tau),
\end{equation}
which matches \Cref{eq:dyn_ca_grad} in the main text.
Hence $\mathcal{L}_{\mathrm{CA}}$ steers $\theta$ along the implicit-classifier gradient $\nabla_{\mathbf{z}_\tau} \log p_{\mathrm{real}}(\mathbf{c}|\mathbf{z}_\tau)$, which corresponds to maximizing text--image alignment.

\subsection{Score-Matching View of the DM Gradient}
\label{app:dm_grad_subsec}

The DM loss in \Cref{eq:ddm_loss} aligns the student-induced distribution $p_{\mathrm{fake}}$ with the real-data distribution $p_{\mathrm{real}}$ by minimizing, in expectation, the reverse Kullback--Leibler divergence
\begin{equation}
    \label{eq:app_kl_grad}
    \nabla_\theta D_{\mathrm{KL}}\!\left(p_{\mathrm{fake}}^{\tilde{\tau}} \,\|\, p_{\mathrm{real}}^{\tilde{\tau}}\right)
    = \mathbb{E}_{\mathbf{z}_{\tilde{\tau}}} \!\left[ \left( \nabla_{\mathbf{z}_{\tilde{\tau}}} \log p_{\mathrm{fake}}(\mathbf{z}_{\tilde{\tau}}|\mathbf{c}) - \nabla_{\mathbf{z}_{\tilde{\tau}}} \log p_{\mathrm{real}}(\mathbf{z}_{\tilde{\tau}}|\mathbf{c}) \right) \frac{\partial \mathbf{z}_{\tilde{\tau}}}{\partial \theta} \right],
\end{equation}
at any continuous noise level $\tilde{\tau} \in (0, 1]$.
Following DMD~\cite{yin2024one}, $p_{\mathrm{fake}}$ is tracked online by the fake teacher $\mathcal{D}_\psi$, which provides a tractable estimate of $\nabla_{\mathbf{z}_{\tilde{\tau}}} \log p_{\mathrm{fake}}$ on the fly.

We now show that the gradient of $\mathcal{L}_{\mathrm{DM}}$ realizes \Cref{eq:app_kl_grad} in score form.
Treating the stop-gradient target in \Cref{eq:ddm_loss} as a constant, the chain rule gives
\begin{equation}
    \label{eq:app_dm_grad}
    \nabla_\theta \mathcal{L}_{\mathrm{DM}} = -w_{\tilde{\tau}} \left( \frac{\partial \mathcal{D}_\theta(\mathbf{x}_{t_i}, t_i, \mathbf{c})}{\partial \theta} \right)^{\!\top} \!\left( \mathcal{D}_\phi(\mathbf{z}_{\tilde{\tau}}, \tilde{\tau}, \mathbf{c}) - \mathcal{D}_\psi(\mathbf{z}_{\tilde{\tau}}, \tilde{\tau}, \mathbf{c}) \right).
\end{equation}
Applying \Cref{prop:tweedie} to $\mathcal{D}_\phi$ and $\mathcal{D}_\psi$ converts the prediction difference into a difference of conditional scores:
\begin{equation}
    \label{eq:app_dm_score}
    \mathcal{D}_\phi(\mathbf{z}_{\tilde{\tau}}, \tilde{\tau}, \mathbf{c}) - \mathcal{D}_\psi(\mathbf{z}_{\tilde{\tau}}, \tilde{\tau}, \mathbf{c}) = \frac{\tilde{\tau}^2}{1-\tilde{\tau}} \!\left( \nabla_{\mathbf{z}_{\tilde{\tau}}} \log p_{\mathrm{real}}(\mathbf{z}_{\tilde{\tau}}|\mathbf{c}) - \nabla_{\mathbf{z}_{\tilde{\tau}}} \log p_{\mathrm{fake}}(\mathbf{z}_{\tilde{\tau}}|\mathbf{c}) \right).
\end{equation}
Substituting into \Cref{eq:app_dm_grad},
\begin{equation}
    \label{eq:app_dm_optim}
    \nabla_\theta \mathcal{L}_{\mathrm{DM}} = -w_{\tilde{\tau}}\,\frac{\tilde{\tau}^2}{1-\tilde{\tau}} \!\left( \frac{\partial \mathcal{D}_\theta(\mathbf{x}_{t_i}, t_i, \mathbf{c})}{\partial \theta} \right)^{\!\top} \!\left( \nabla_{\mathbf{z}_{\tilde{\tau}}} \log p_{\mathrm{real}}(\mathbf{z}_{\tilde{\tau}}|\mathbf{c}) - \nabla_{\mathbf{z}_{\tilde{\tau}}} \log p_{\mathrm{fake}}(\mathbf{z}_{\tilde{\tau}}|\mathbf{c}) \right),
\end{equation}
which matches \Cref{eq:dyn_dm_grad} in the main text.
The bracketed score difference shares the structure of the integrand in \Cref{eq:app_kl_grad} and, in expectation, drives $p_{\mathrm{fake}}$ toward $p_{\mathrm{real}}$ at noise level $\tilde{\tau}$.

Together, \Cref{eq:app_cfg_optim,eq:app_dm_optim} share the same prefactor $\tfrac{\tau^2}{1-\tau}\,(\partial \mathcal{D}_\theta / \partial \theta)^{\!\top}$ and act on a score-valued target.
Because the dynamic continuous schedule (\Cref{sec:pilot_study}) draws the student anchor $t_i$ and the teacher perturbation timesteps $\tau, \tilde{\tau}$ independently from the same continuous distribution on $(0, 1]$, the resulting supervision is applied uniformly over the entire time domain rather than only at sparse discrete anchors.

\section{Local and Global Truncation Error of Euler Sampling}
\label{app:euler_error}

This appendix derives the local and global truncation error bounds for explicit Euler sampling to formally show that the error is controlled by $M_2$ (the supremum of the velocity field's material derivative), which is precisely the quantity suppressed by the CDM loss.

\paragraph{Setup and Local Error}
Consider the probability-flow ODE $d\mathbf{x}_\tau / d\tau = v_\theta(\mathbf{x}_\tau, \tau, \mathbf{c})$ for $\tau \in [\epsilon, 1]$. Discretizing the interval with a schedule $1 = t_1 > t_2 > \cdots > t_{N} = \epsilon$ and step sizes $\Delta t = h_j := t_j - t_{j+1} > 0$, a single explicit Euler step from $t_j$ down to $t_{j+1}$ is:
\begin{equation}
    \label{eq:app_euler_step}
    \tilde{\mathbf{x}}_{t_{j+1}} = \mathbf{x}_{t_j} - h_j\, v_\theta(\mathbf{x}_{t_j}, t_j, \mathbf{c}).
\end{equation}
Assuming $v_\theta$ is sufficiently smooth and $L$-Lipschitz in $\mathbf{x}$, we follow the standard local-error convention by comparing a single Euler step against the exact ODE solution passing through the same starting point $\mathbf{x}_{t_j}$ at $\tau = t_j$.
Expanding the true solution around $\tau = t_j$ via Taylor's theorem with Lagrange remainder gives:
\begin{equation}
    \mathbf{x}_{t_{j+1}} = \mathbf{x}_{t_j} - h_j\, \dot{\mathbf{x}}_{t_j} + \tfrac{1}{2} h_j^2\, \ddot{\mathbf{x}}_{\xi_j}, \qquad \xi_j \in (t_{j+1}, t_j).
\end{equation}
Substituting $\dot{\mathbf{x}}_\tau = v_\theta(\mathbf{x}_\tau, \tau, \mathbf{c})$ and subtracting \Cref{eq:app_euler_step} yields the local truncation error:
\begin{equation}
    \label{eq:app_lte}
    \big\| \mathbf{x}_{t_{j+1}} - \tilde{\mathbf{x}}_{t_{j+1}} \big\|
    \;=\; \tfrac{1}{2} h_j^2\, \|\ddot{\mathbf{x}}_{\xi_j}\|
    \;\le\; \tfrac{1}{2} h_j^2 \cdot M_2^{(j)} \;=\; \mathcal{O}\left( (\Delta t)^2 \sup_{\tau \in [t_{j+1}, t_j]} \Big\| \frac{d v_\theta}{d\tau} \Big\| \right),
\end{equation}
where $M_2^{(j)}$ bounds the material derivative $d v_\theta / d\tau$ (the total variation of $v_\theta$ along the trajectory) over the step:
\begin{equation}
    \label{eq:app_m2}
    M_2^{(j)} \;:=\; \sup_{\tau \in [t_{j+1}, t_j]} \Big\| \frac{d v_\theta}{d\tau} \Big\| \;=\; \sup_{\tau \in [t_{j+1}, t_j]} \Big\| \partial_\tau v_\theta(\mathbf{x}_\tau, \tau, \mathbf{c}) \;+\; J_{\mathbf{x}} v_\theta(\mathbf{x}_\tau, \tau, \mathbf{c})\, v_\theta(\mathbf{x}_\tau, \tau, \mathbf{c}) \Big\|.
\end{equation}

\paragraph{Global Error}
Accumulating this local error over $N$ steps down to $\tau = \epsilon$ yields the global error bound:
\begin{equation}
    \label{eq:app_global}
    \big\| \mathbf{x}_\epsilon - \hat{\mathbf{x}}_\epsilon \big\|
    \;\le\; \frac{e^{L(1-\epsilon)} - 1}{L} \cdot \tfrac{1}{2}\, \bar{h}\, M_2,
    \qquad \bar{h} := \max_j h_j, \quad M_2 := \max_j M_2^{(j)}.
\end{equation}
This bound shows that the global error is $\mathcal{O}(\bar{h})$ (since accumulating $\mathcal{O}(1/\bar{h})$ steps of $\mathcal{O}(\bar{h}^2)$ local error loses one order), and for a fixed step budget $N$, $M_2$ is the only factor that can be optimized by training.

\paragraph{How CDM Constrains $M_2$}
The CDM loss supervises the student at the extrapolated off-trajectory latent:
\begin{equation}
    \mathbf{x}_{t_i'} \;=\; \mathbf{x}_{t_i} + \Delta t\, v_\theta(\mathbf{x}_{t_i}, t_i, \mathbf{c}), \quad \text{where} \;\; \Delta t = t_i' - t_i.
\end{equation}
To understand its regularization effect, consider the first-order Taylor expansion of the velocity field at this extrapolated point around $(\mathbf{x}_{t_i}, t_i)$:
\begin{align}
    v_\theta(\mathbf{x}_{t_i'}, t_i') 
    &\approx v_\theta(\mathbf{x}_{t_i}, t_i) + \partial_\tau v_\theta(\mathbf{x}_{t_i}, t_i, \mathbf{c}) \Delta t + J_{\mathbf{x}} v_\theta(\mathbf{x}_{t_i}, t_i, \mathbf{c}) (\mathbf{x}_{t_i'} - \mathbf{x}_{t_i}) \nonumber \\
    &= v_\theta(\mathbf{x}_{t_i}, t_i) + \Delta t \underbrace{\left( \partial_\tau v_\theta + J_{\mathbf{x}} v_\theta\, v_\theta \right)}_{= \, d v_\theta / d\tau}.
\end{align}
Rearranging the terms reveals that the material derivative is functionally approximated by the finite difference taking the Euler leap:
\begin{equation}
    \frac{d v_\theta}{d\tau} \approx \frac{v_\theta(\mathbf{x}_{t_i'}, t_i') - v_\theta(\mathbf{x}_{t_i}, t_i)}{\Delta t}.
\end{equation}
While standard distillation ensures the student matches the teacher at the anchor $\mathbf{x}_{t_i}$ (i.e., $v_\theta(\mathbf{x}_{t_i}, t_i) \approx v_\phi(\mathbf{x}_{t_i}, t_i)$), the CDM loss additionally enforces $v_\theta(\mathbf{x}_{t_i'}, t_i') \approx v_\phi(\mathbf{x}_{t_i'}, t_i')$. Together, they force the student's material derivative to mimic that of the teacher $\frac{d v_\theta}{d\tau} \approx \frac{d v_\phi}{d\tau}$. Since the pre-trained teacher naturally presents a smooth velocity field with bounded variation, CDM effectively transfers this smoothness to the student, implicitly regularizing $M_2$ and preventing the sporadic high-frequency oscillation typically observed in discretely trained models.

\section{Training Algorithm}
\label{sec:appendix_algorithm}

For completeness, we summarize the full training procedure of CDM in \Cref{alg:training}.
The pseudocode reflects the implementation details described in \Cref{sec:method}: the dynamic continuous time schedule (\Cref{sec:pilot_study}), the decoupled CA and DM losses on backward-simulation anchors (\Cref{sec:prelim}), and the off-trajectory CDM loss anchored to the localized prediction $\hat{\mathbf{x}}_0^{(i')}$ (\Cref{sec:cdm}).
All three student-side losses are computed within a single backward simulation per iteration; the intermediate latent $\mathbf{x}_{t_i}$ is extracted from the trajectory at a uniformly sampled anchor $t_i$. As established in D-DMD~\cite{liu2025decoupled}, this single shared latent $\mathbf{x}_{t_i}$ is used across all objectives to improve training efficiency, and the student parameters $\theta$ receive a single combined gradient step.
The fake teacher $\psi$ is updated on a separate optimizer using the standard flow-matching objective on the re-noised one-step student prediction $\hat{\mathbf{x}}_0^{(i)}$, sharing the same anchor $i$ with the distillation losses.

\begin{algorithm}[htbp]
\caption{Training Procedure of CDM}
\label{alg:training}
\begin{algorithmic}[1]
\REQUIRE Student $\theta$, real teacher $\phi$ (frozen), fake teacher $\psi$, max length $N_{\max}$, guidance scale $\alpha$, prompt set $\mathcal{C}$
\REPEAT
    \STATE Sample prompt $\mathbf{c} \sim \mathcal{C}$,\; $N \sim \mathcal{U}\{1, \ldots, N_{\max}\}$,\; schedule $1 = t_1 > t_2 > \cdots > t_N > 0$\\[2pt]

    \STATE \textcolor{gray}{\textit{// 1. Backward simulation (no gradient)}}
    \STATE $\mathbf{x}_{t_1} \sim \mathcal{N}(\mathbf{0}, \mathbf{I})$;\; run $\theta$ for $N$ Euler steps to obtain $\{\mathbf{x}_{t_n}\}_{n=1}^N$
    \STATE Sample anchor $i \sim \mathcal{U}\{1, \ldots, N\}$\\[2pt]

    \STATE \textcolor{gray}{\textit{// 2. Fake teacher update (on the re-noised one-step student prediction $\hat{\mathbf{x}}_0^{(i)}$)}}
    \STATE $\hat{\mathbf{x}}_0^{(i)} \leftarrow \mathcal{D}_\theta(\mathbf{x}_{t_i}, t_i, \mathbf{c})$
    \STATE Sample $\tau_\psi \sim \mathcal{U}(0, 1]$;\; $\mathbf{z}_{\tau_\psi} \leftarrow (1 - \tau_\psi)\,\operatorname{sg}[\hat{\mathbf{x}}_0^{(i)}] + \tau_\psi\,\boldsymbol{\epsilon}_\psi$
    \STATE $\psi \leftarrow \psi - \eta_\psi\, \nabla_\psi \bigl\| v_\psi(\mathbf{z}_{\tau_\psi}, \tau_\psi, \mathbf{c}) - (\boldsymbol{\epsilon}_\psi - \operatorname{sg}[\hat{\mathbf{x}}_0^{(i)}]) \bigr\|_2^2$\\[2pt]

    \STATE \textcolor{gray}{\textit{// 3. CA Loss}}
    \STATE Sample $\tau \sim \mathcal{U}(0, 1]$;\; $\mathbf{z}_\tau \leftarrow (1-\tau)\,\operatorname{sg}[\hat{\mathbf{x}}_0^{(i)}] + \tau\,\boldsymbol{\epsilon}_\tau$
    \STATE Compute $\mathcal{L}_{\mathrm{CA}}$ \hfill (\Cref{eq:cfg_loss})\\[2pt]

    \STATE \textcolor{gray}{\textit{// 4. DM Loss}}
    \STATE Sample $\tilde{\tau} \sim \mathcal{U}(0, 1]$;\; $\mathbf{z}_{\tilde{\tau}} \leftarrow (1-\tilde{\tau})\,\operatorname{sg}[\hat{\mathbf{x}}_0^{(i)}] + \tilde{\tau}\,\boldsymbol{\epsilon}_{\tilde{\tau}}$
    \STATE Compute $\mathcal{L}_{\mathrm{DM}}$ \hfill (\Cref{eq:ddm_loss})\\[2pt]

    \STATE \textcolor{gray}{\textit{// 5. CDM Loss}}
    \STATE Sample $t_i' \sim \mathcal{U}(0, 1]$;\; $\mathbf{x}_{t_i'} \leftarrow \mathbf{x}_{t_i} + (t_i' - t_i)\, v_\theta(\mathbf{x}_{t_i}, t_i, \mathbf{c})$
    \STATE $\hat{\mathbf{x}}_0^{(i')} \leftarrow \mathcal{D}_\theta(\mathbf{x}_{t_i'}, t_i', \mathbf{c})$
    \STATE Sample $\hat{\tau} \sim \mathcal{U}(0, 1]$;\; $\mathbf{z}_{\hat{\tau}} \leftarrow (1 - \hat{\tau})\,\operatorname{sg}[\hat{\mathbf{x}}_0^{(i')}] + \hat{\tau}\,\boldsymbol{\epsilon}_{\hat{\tau}}$
    \STATE Compute $\mathcal{L}_{\mathrm{CDM}}$ \hfill (\Cref{eq:cdm_base_loss})\\[2pt]

    \STATE \textcolor{gray}{\textit{// 6. Student update}}
    \STATE $\theta \leftarrow \theta - \eta\, \nabla_\theta ( \mathcal{L}_{\mathrm{CA}} + \mathcal{L}_{\mathrm{DM}} + \mathcal{L}_{\mathrm{CDM}} )$ \hfill (\Cref{eq:full_objective})
\UNTIL{converged}
\end{algorithmic}
\end{algorithm}

\section{More Quantitative Comparisons}
\label{app:appendix_efficiency}

To further evaluate model performance and training/inference resource consumption, we select the strongest baseline from \Cref{tab:metrics_comparison}, namely D-DMD~\cite{liu2025decoupled}, and conduct a more in-depth comparison against our CDM along two complementary dimensions: (i) additional quality metrics that are not covered by the main table (OCR accuracy for text rendering and FID for distributional fidelity), and (ii) training and inference efficiency.

\paragraph{Evaluation Protocol}
For text rendering evaluation, we calculate OCR accuracy using PaddleOCR~\cite{cui2025paddleocr} on a test set of 1K OCR prompts from FlowGRPO~\cite{liu2025flow}.
Additionally, we compute the Fr\'echet Inception Distance (FID)~\cite{heusel2017gans} using 10K prompts from the COCO 2014 validation set~\cite{lin2014microsoft}.
Relative training time is normalized to D-DMD, and inference latency is measured at $1024 \times 1024$ resolution with 4 NFE on a single GPU under identical hardware and software conditions across all methods.

\paragraph{Results}
As reported in \Cref{tab:training_efficiency}, CDM delivers the strongest overall quality among the three configurations.
Across the seven quality metrics, CDM ranks first on six of them (Aesthetic, DPGBench, PickScore, HPSv3, CLIPScore, and FID), while remaining closely competitive behind the fixed-schedule variant on OCR accuracy.
The fixed-schedule variant of CDM consistently ranks second across most quality metrics, indicating that our overall framework is already a strong recipe and that the dynamic time schedule provides a further, consistent improvement on top of it.
On the efficiency side, our dynamic continuous-time schedule and CDM loss introduce additional per-iteration training overhead---primarily due to the longer average backward simulation length and the extra forward pass on the extrapolated off-trajectory latent $\mathbf{x}_{t_i'}$---resulting in a relative training time of roughly $1.8\times$ that of D-DMD with comparable peak memory ($62.5$ vs.\ $62.2$~GB).
Crucially, this overhead is confined to the training phase: at inference time, all three configurations share identical computational cost since they use the same backbone and the same number of function evaluations, so the per-image latency of CDM is on par with D-DMD ($246$~ms/img).

\begin{table}[ht]
\centering
\caption{\textbf{Extended comparison with the strongest baseline D-DMD on SD3-Medium with 4 NFE.} We report the five main quality metrics (Aesthetic, DPGBench, PickScore, HPSv3, CLIPScore) together with two complementary quality metrics (OCR accuracy on 1K FlowGRPO prompts and FID on 10K COCO 2014 val prompts), as well as training memory, relative training time (normalized to D-DMD), and inference latency. The best and \underline{second-best} results in each column are highlighted in \textbf{bold} and \underline{underline}, respectively.}
\label{tab:training_efficiency}
\setlength{\tabcolsep}{3pt}
\resizebox{\textwidth}{!}{
\begin{tabular}{l ccccc cc ccc}
\toprule
 & \multicolumn{5}{c}{\textbf{Quality (Main Metrics)}}
 & \multicolumn{2}{c}{\textbf{Quality (Extra)}}
 & \multicolumn{3}{c}{\textbf{Efficiency}} \\
\cmidrule(lr){2-6}\cmidrule(lr){7-8}\cmidrule(lr){9-11}
\textbf{Method}
 & \textbf{Aes}$\uparrow$ & \textbf{DPG}$\uparrow$ & \textbf{Pick}$\uparrow$
 & \textbf{HPSv3}$\uparrow$ & \textbf{CLIP}$\uparrow$
 & \textbf{OCR}$\uparrow$ & \textbf{FID}$\downarrow$
 & \textbf{Mem (GB)} & \textbf{TrainTime} & \textbf{Latency (ms)} \\
\midrule
D-DMD$^{*}$~\cite{liu2025decoupled}  & 6.038 & \underline{84.52} & 21.85 & 9.176 & 27.69 & 33.47 & 31.47 & 62.2 & 1.0$\times$ & 246 \\
CDM (w/ Fixed Schedule)              & \underline{6.051} & 83.84 & \underline{21.89} & \underline{9.482} & \underline{27.75} & \textbf{37.33} & \underline{31.05} & 62.5 & 1.1$\times$ & 246 \\
\rowcolor[HTML]{EBE4F2}
CDM                                  & \textbf{6.075} & \textbf{85.26} & \textbf{21.95} & \textbf{9.561} & \textbf{27.98} & \underline{34.82} & \textbf{30.30} & 62.5 & 1.8$\times$ & 246 \\
\bottomrule
\end{tabular}
}
\end{table}

\section{Quantitative Evaluation of the DM Loss}
\label{app:dm_loss_analysis}

As discussed in \Cref{sec:intro} and visually demonstrated in \Cref{fig:cfg_vs_nocfg}, relying solely on the Distribution Matching (DM) objective limits the student model to learning a marginal, unguided distribution.
To quantitatively validate this, we evaluate the performance of the student models distilled exclusively with the DM loss and compare them against their respective teacher models running with and without Classifier-Free Guidance (CFG).

The results are summarized in \Cref{tab:dm_loss_analysis}.
When disabling CFG for the teacher (Teacher CFG-free), we observe a significant deterioration in both semantic alignment metrics (e.g., DPGBench, HPSv3) and visual fidelity.
Crucially, the student model distilled with the DM loss alone almost completely mirrors this performance drop, attaining metric scores that closely track the CFG-free teacher.
As shown in \Cref{tab:dm_loss_analysis}, this highly consistent degradation pattern on both SD3-Medium and Longcat-Image confirms that the DM loss accurately aligns the student with the underlying teacher---but strictly with its weak, unguided marginal distribution.

\begin{table}[h]
\centering
\caption{\textbf{Quantitative validation of the DM loss's alignment with CFG-free distributions.} The student model distilled exclusively with the DM loss heavily mirrors the performance deterioration of the CFG-free teacher across all metrics, confirming our visual observations in \Cref{fig:cfg_vs_nocfg}.}
\label{tab:dm_loss_analysis}
\vspace{2mm}
\resizebox{\linewidth}{!}{
\begin{tabular}{lc ccccc}
\toprule
\textbf{Method} & \textbf{NFE} & \textbf{Aesthetic}$\uparrow$ & \textbf{DPGBench}$\uparrow$ & \textbf{PickScore}$\uparrow$ & \textbf{HPSv3}$\uparrow$ & \textbf{CLIPScore}$\uparrow$ \\
\midrule
\multicolumn{7}{l}{\textit{\color{gray}SD3-Medium}} \\
\rowcolor[HTML]{E6E6E6} 
Teacher (with CFG) & 100 & 5.885 & 85.04 & 21.73 & 8.189 & 28.60 \\
Teacher (CFG-free)              & 50  & 5.681 & 69.65 & 20.24 & 4.693 & 24.32 \\
Student (DM loss only)          & 4   & 5.787 & 70.60 & 20.82 & 7.258 & 25.31 \\
\midrule
\multicolumn{7}{l}{\textit{\color{gray}Longcat-Image}} \\
\rowcolor[HTML]{E6E6E6} 
Teacher (with CFG)              & 100 & 5.926 & 87.08 & 21.65 & 9.450 & 26.78 \\
Teacher (CFG-free)              & 50  & 5.952 & 84.72 & 21.12 & 8.253 & 25.63 \\
Student (DM loss only)          & 4   & 6.028 & 80.53 & 20.10 & 8.186 & 21.69 \\
\bottomrule
\end{tabular}
}
\end{table}

\section{CDM under Varying Inference Steps}
\label{sec:appendix_nfe_robustness}

Although our student is distilled with a target of 4 NFE, the continuous-time training paradigm of CDM does not bind the resulting model to any specific inference schedule.
On the one hand, the dynamic continuous schedule (\Cref{sec:pilot_study}) randomizes the backward simulation length $N \sim \mathcal{U}\{1, N_{\max}\}$ at every training iteration, so the student is exposed to trajectories of varying lengths rather than a single fixed grid.
On the other hand, the $\mathcal{L}_{\mathrm{CDM}}$ loss (\Cref{sec:cdm}) further regularizes the student's velocity field $v_\theta$ across the continuous time domain, which directly suppresses the per-step truncation error of order $\mathcal{O}((\Delta t)^2 \sup_\tau \|dv_\theta/d\tau\|)$ that dominates few-step Euler integration.
As a result, CDM remains usable across a range of NFEs at test time without any retraining or schedule-specific tuning.
\Cref{fig:step_robustness} shows samples generated by the same CDM checkpoint under NFE $\in \{3, 4, 6, 8\}$ with identical prompts and seeds, where the model produces coherent and well-aligned images throughout the range and progressively recovers finer details as the NFE increases.

\begin{figure}[t]
    \centering
    \includegraphics[width=\textwidth]{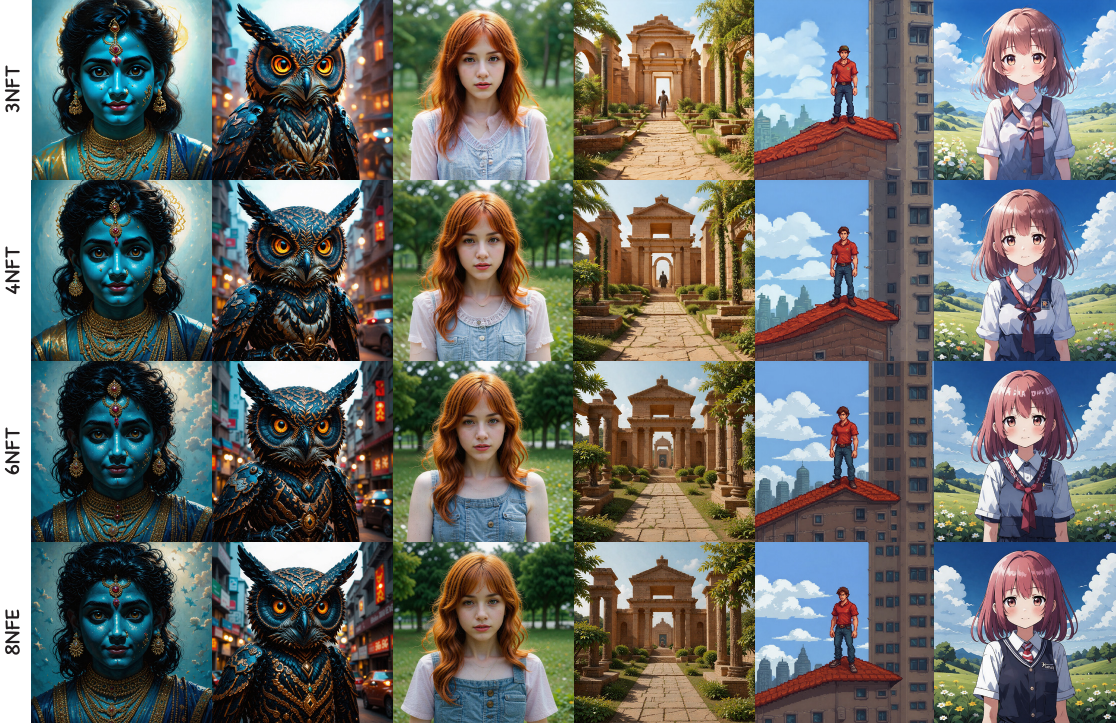}
    \caption{Generations from the same CDM checkpoint under varying NFE $\in \{3, 4, 6, 8\}$, using identical prompts and random seeds across columns. CDM produces coherent and prompt-aligned images across the full range, with finer details emerging as more inference steps are used.}
    \label{fig:step_robustness}
\end{figure}

\section{More Qualitative Results}
\label{sec:appendix_more_qualitative}

To complement the main qualitative comparison in \Cref{fig:main_results}, we provide additional samples generated by CDM on both backbones.
\Cref{fig:appendix_more_sd3} shows results on SD3-Medium, covering a diverse set of prompts that span photorealistic portraits, complex scenes, stylized illustrations, and text-rich compositions.
\Cref{fig:appendix_more_longcat} shows the corresponding results on Longcat-Image, demonstrating that the proposed continuous-time distribution matching framework generalizes consistently across different backbones.
All images are generated with 4 NFE.

\begin{figure}[p]
    \centering
    \includegraphics[width=\textwidth]{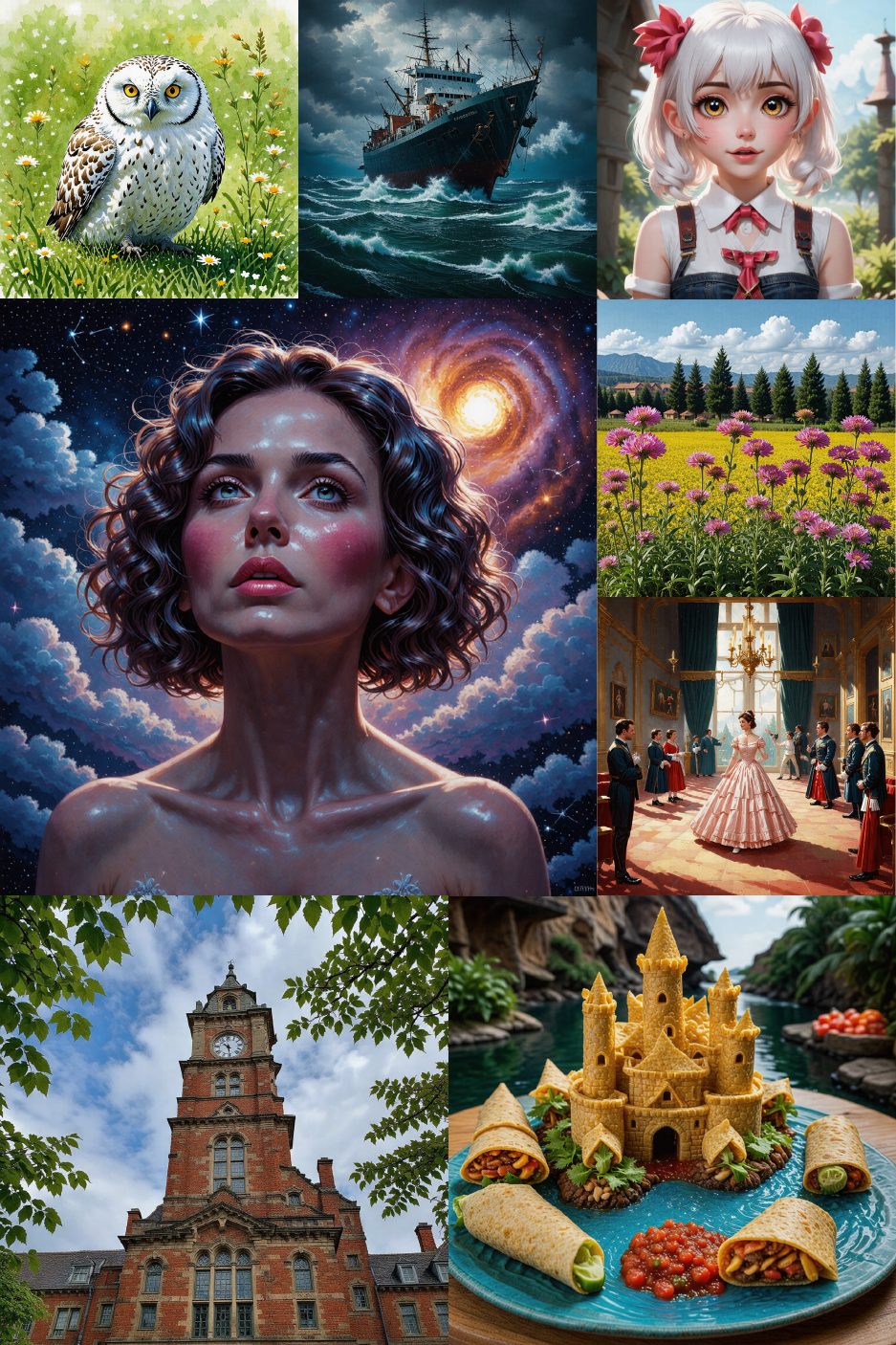}
    \caption{Additional qualitative results of CDM on SD3-Medium at $1024 \times 1024$ resolution with 4 NFE, covering diverse prompt categories. Zoom in for best view.}
    \label{fig:appendix_more_sd3}
\end{figure}

\begin{figure}[p]
    \centering
    \includegraphics[width=\textwidth]{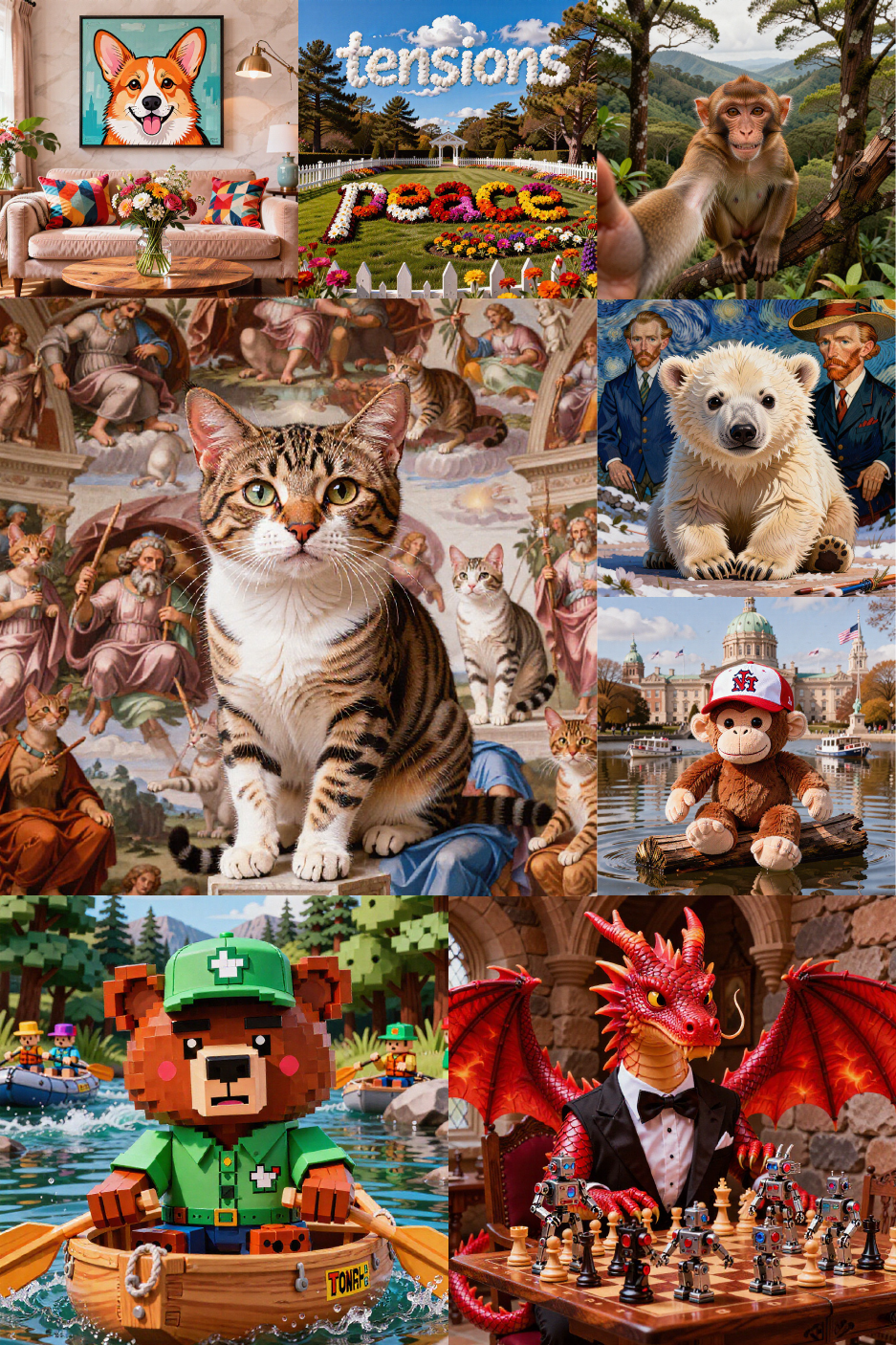}
    \caption{Additional qualitative results of CDM on Longcat-Image at $1024 \times 1024$ resolution with 4 NFE, covering diverse prompt categories. Zoom in for best view.}
    \label{fig:appendix_more_longcat}
\end{figure}

\end{document}